\documentclass[runningheads]{llncs}

\usepackage[table, dvipsnames]{xcolor}
\usepackage{wrapfig}
\usepackage{multicol}
\usepackage{cite}

\usepackage[bookmarks=true]{hyperref}

\usepackage{times}
\usepackage{adjustbox}
\usepackage{floatrow}
\usepackage{amsmath} % assumes amsmath package installed

\usepackage{amssymb}  % assumes amsmath package installed
\usepackage{amsthm}
\usepackage{bm}
\usepackage{todonotes}
\usepackage[linesnumbered,ruled,noend]{algorithm2e}
\usepackage{bmpsize}
\usepackage{chngcntr}
\usepackage{apptools}
\AtAppendix{\counterwithin{theorem}{section}}
\AtAppendix{\counterwithin{corollary}{section}}
\AtAppendix{\counterwithin{definition}{section}}
\AtAppendix{\counterwithin{lemma}{section}}
\newcommand{\X}{\mathcal{X}}
\newcommand{\U}{\mathcal{U}}
\newcommand{\Y}{\mathcal{Y}}

\renewcommand{\S}{\mathcal{S}}
\newcommand{\V}{\mathcal{V}}
\newcommand{\B}{\mathcal{B}}
\newcommand{\Xsafe}{\X_{\textrm{safe}}}
\newcommand{\Xunsafe}{\X_{\textrm{unsafe}}}
\newcommand{\T}{\mathcal{T}}

\newcommand{\I}{\mathbf{I}}
\newcommand{\f}{f}

\newcommand{\h}{h}
\newcommand{\gzero}{f}
\newcommand{\gone}{B}

\newcommand{\lmtd}{\textrm{CORRT} }
\newcommand{\lmtds}{\textrm{CORRT}}

\newcommand{\yred}{y_r}

\newcommand{\dx}{w_x}
\newcommand{\dy}{w_y}
\newcommand{\bardy}{\bar{w}_y}
\newcommand{\bardx}{\bar{w}_x}
\newcommand{\dq}{w_q}
\newcommand{\de}{w_e}
\newcommand{\dc}{w_c}
\newcommand{\distc}{d_c}
\newcommand{\diste}{d_e}
\newcommand{\Mc}{M_c}
\newcommand{\Wc}{W_c}
\newcommand{\Mo}{M_e}
\newcommand{\Wo}{W_e}

\newcommand{\Cred}{C_r}
\newcommand{\disp}{\mathcal{R}}
\newcommand{\hinv}{\hat{h}^{-1}}
\newcommand{\err}{\epsilon}

\newcommand{\Ndata}{N_\textrm{data}}

\newcommand{\xhatdot}{\dot{\hat{x}}}
\newcommand{\xhat}{\hat{x}}
\newcommand{\maxeig}[1]{\bar\lambda(#1)}
\newcommand{\mineig}[1]{\underline\lambda(#1)}
\newcommand{\maxeigc}[1]{\bar\lambda_{D_c}(#1)}
\newcommand{\mineigc}[1]{\underline\lambda_{D_c}(#1)}
\newcommand{\maxeige}[1]{\bar\lambda_{D_e}(#1)}

\newcommand{\maxsing}[1]{\bar\sigma(#1)}

\usepackage{scalerel,stackengine}
\stackMath
\newcommand\reallywidehat[1]{%
\savestack{\tmpbox}{\stretchto{%
  \scaleto{%
    \scalerel*[\widthof{\ensuremath{#1}}]{\kern-.6pt\bigwedge\kern-.6pt}%
    {\rule[-\textheight/2]{1ex}{\textheight}}%WIDTH-LIMITED BIG WEDGE
  }{\textheight}% 
}{0.5ex}}%
\stackon[1pt]{#1}{\tmpbox}%
}

\usepackage[T1]{fontenc}
\def\doi#1{\href{https://doi.org/\detokenize{#1}}{\url{https://doi.org/\detokenize{#1}}}}
\usepackage{listings}
\begin{document}
\title{Safe Output Feedback Motion Planning from Images via Learned Perception Modules and Contraction Theory}
\titlerunning{Safe Output Feedback Motion Planning from Images}
\author{Glen Chou\and
Necmiye Ozay \and
Dmitry Berenson%
\thanks{\footnotesize{This research was supported in part by NSF grants IIS-1750489 and IIS-2113401, ONR grants N00014-21-1-2118, N00014-18-1-2501 and N00014-21-1-2431, and a National Defense Science and Engineering Graduate (NDSEG) fellowship.}}}
\authorrunning{G. Chou et al.}
\institute{Dept. of Electrical Engineering and Computer Science, \\ University of Michigan, Ann Arbor, MI 48105, USA\\
\email{\{gchou, necmiye, dmitryb\}@umich.edu}}
\maketitle              % typeset the header of the contribution
%
%\vspace{-20pt}

\begin{abstract}
We present a motion planning algorithm for a class of uncertain control-affine nonlinear systems which guarantees runtime safety and goal reachability when using high-dimensional sensor measurements (e.g., RGB-D images) and a learned perception module in the feedback control loop. First, given a dataset of states and observations, we train a perception system that seeks to invert a subset of the state from an observation, and estimate an upper bound on the perception error which is valid with high probability in a trusted domain near the data. Next, we use contraction theory to design a stabilizing state feedback controller and a convergent dynamic state observer which uses the learned perception system to update its state estimate. We derive a bound on the trajectory tracking error when this controller is subjected to errors in the dynamics and incorrect state estimates. Finally, we integrate this bound into a sampling-based motion planner, guiding it to return trajectories that can be safely tracked at runtime using sensor data. We demonstrate our approach in simulation on a 4D car, a 6D planar quadrotor, and a 17D manipulation task with RGB(-D) sensor measurements, demonstrating that our method safely and reliably steers the system to the goal, while baselines that fail to consider the trusted domain or state estimation errors can be unsafe.

\end{abstract}

\vspace{-25pt}
\section{Introduction}
\vspace{-22pt}

\begin{figure}
\includegraphics[width=\textwidth]{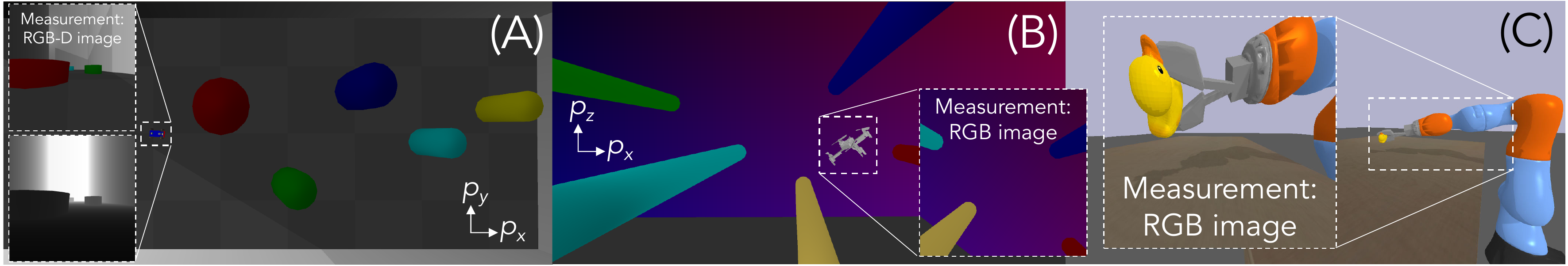}
\vspace{-8pt}
\caption{For a 4D car, a 6D quadrotor, and a 14D arm, we compute plans that can be safely stabilized to reach goals at runtime using rich sensor observations in the form of RGB(-D) images.} \label{fig:environments}
\end{figure}\vspace{-16pt}

Safely and reliably deploying an autonomous robot requires a systematic analysis of the uncertainties that it may face across its perception, planning, and feedback control modules. State-of-the-art methods largely analyze each module separately; e.g., by first certifying perception \cite{DBLP:journals/trob/YangSC21}, finding a safe plan under a nominal dynamics model \cite{lavalle2006planning}, and then using a stable tracking controller \cite{sumeet_icra}. However, this ignores how the errors in each module can propagate. Inaccuracies in the dynamics and perception can destabilize the downstream feedback controller and lead to failure, revealing a need to unify  perception, planning, and control to guarantee safety for the end-to-end autonomy pipeline.

To address this gap, we consider one such unified approach: the Output Feedback Motion Planning problem (OFMP) \cite{DBLP:journals/corr/abs-2002-02928}, which jointly plans nominal trajectories and designs feedback controllers which safely stabilize the system to some goal when using imperfect state information (i.e., output feedback). A concrete way to solve the OFMP is to bound the set of states that the system may reach while tracking a plan using output feedback, that is, a closed-loop output feedback trajectory tracking tube, and ensure it is collision-free. Practical robots present challenges in solving the OFMP:

\begin{enumerate}
	\item The tracking tubes should be efficiently computable for arbitrary trajectories so that they can be used in the planning loop to restrict the set of states that can be safely visited. However, solving this reachability problem is computationally demanding.
	\item Processing rich sensor data (e.g., images, depth maps, etc.) at runtime is often done via deep learning-based perception modules, which are powerful but error-prone. Bounding this error and bounding its effect on trajectory tracking error is difficult.
\end{enumerate}

To address the first challenge, we use contraction theory, which is of specific interest for the OFMP as it enables the 1) design of stabilizing feedback controllers \cite{manchester} and convergent state estimators \cite{DBLP:journals/tac/DaniCH15} and 2) fast computation of tracking/estimation tubes, given a bound on the disturbances that the controller and observer are subjected to \cite{DBLP:conf/aucc/ManchesterS14}. Estimating this bound is central to our solution of the second challenge, where we use data to 1) estimate a bound on the error of a learned perception module which is valid with high probability and 2) bound the level to which incorrect state estimates can destabilize the controller. Combining these solutions provides accurate tubes that can be used in planning. In summary, we develop a contraction-based output feedback motion planning algorithm for control-affine systems stabilized from image observations, which retains guarantees on safety and goal reachability. Our specific contributions are:

\begin{itemize}
	\item A learning-based framework for integrating high-dimensional observations into contraction-based control and estimation that can generalize across environments
	\item A trajectory tracking error bound for contraction-based feedback controllers in output feedback, subjected to a disturbance that accurately reflects the perception error
	\item A sampling-based planner which solves the OFMP, returning plans that can be safely tracked and that reliably reach the goal at runtime using image observations
	\item Validation in simulation on a 4D nonholonomic car, a 6D planar quadrotor, and a 17D manipulation task, guaranteeing safety whereas baseline approaches fail
\end{itemize}

\section{Related Work}

First, our work is related to and draws from contraction-based control and estimation. Contraction-based robust control \cite{sumeet_icra, cdc, DBLP:journals/ral/TsukamotoC21} can ensure safety for uncertain systems using perfect state feedback, but the guarantees are lost if using imperfect state estimates. Other work has studied contraction-based convergence guarantees for state estimation without control input, e.g., \cite{DBLP:journals/csysl/TsukamotoC21,DBLP:journals/tac/BonnabelS15,DBLP:journals/tac/DaniCH15}; however, solving the OFMP requires jointly analyzing the controller and observer. Most closely related is \cite{DBLP:conf/aucc/ManchesterS14}, which studies output feedback control via contracting controllers and observers; however, it considers simple measurement models and does not derive the tracking tubes needed for the OFMP.

Our work also relates to control from rich observations. Differentiable filtering \cite{DBLP:journals/arobots/KlossMB21} learns state estimators from images in an end-to-end fashion, which while empirically successful, do not provide guarantees. Other work focuses on safety: \cite{DBLP:conf/l4dc/DeanMRY20} safely controls linear systems using learned observation maps; other methods use Control Barrier/Lyapunov Functions (CBF/CLFs) to guarantee safety for nonlinear systems by robustifying the CBF condition to measurement errors \cite{dean, DBLP:conf/iros/CosnerSTMBA21}; however, these methods use simple sensor models or require that the entire state is invertible from one observation, precluding their use on states that must be estimated over time, e.g., velocities. In contrast, our method only seeks to invert a \textit{subset} of the state, which is then used in a dynamic observer to estimate the unobserved states. Other work \cite{dawson} combines CLFs and CBFs to safely reach goals from observations, but focuses on simpler LiDAR sensor models.

Finally, our work relates to planning under uncertainty. Funnel-based methods buffer motion primitives with tracking tubes under perfect \cite{MajumdarT17} and vision-based \cite{sushant} feedback control. In contrast, we do not rely on precomputed primitives, and can plan novel trajectories. Other methods \cite{firm, DBLP:conf/iros/BahreinianMT21} consider measurement error in planning but are either restricted to linear systems or simple sensor models. These methods are instances of the generally intractable belief-space planning problem; solving this problem requires simplifying assumptions \cite{sunberg} that may compromise safety. We do not solve the full belief-space planning problem; instead of tracking belief distributions, our set-based approach bounds the reachable states and state estimates under the worst-case error. 

\section{Preliminaries and Problem Statement}

We consider uncertain continuous-time control-affine nonlinear systems (which include many common mechanical systems of interest \cite{lavalle2006planning}) with output observations
\begin{subequations}\label{eq:sys_and_obs}
	\begin{align}
     \dot x(t) &= \f(x(t)) + Bu(t) + B_w(t)\dx(t)\label{eq:system}\\
     y(t) &= \h(x(t), \theta) + B_y\dy(t)\label{eq:observation}
\end{align} 
\end{subequations}
where $\f: \X \rightarrow \X$, $\X \subseteq \mathbb{R}^{n_x}$, $B \in \mathbb{R}^{n_x \times n_u}$, $B_w: [0,\infty)\rightarrow \mathbb{R}^{n_x \times n_{w_x}}$, $B_y \in \mathbb{R}^{n_y \times n_{w_y}}$, $\U \subseteq \mathbb{R}^{n_u}$, and $\dx \in \mathbb{R}^{n_{w_x}}$ is a possibly stochastic state disturbance where $\Vert \dx(t) \Vert \le \bardx$, for all $t$. Without loss of generality, we assume $\Vert B_w(t)\Vert \le 1$, for all $t$. Norms $\|\cdot\|$ without subscript are the (induced) 2-norm. We obtain high-dimensional observations $y \in \Y \subseteq \mathbb{R}^{n_y}$ (e.g., $N\times N$-pixel RGB-D images, leading to $n_y = 4N^2$), generated by a deterministic, nonlinear function $\h(x, \theta): \X \times \Theta \rightarrow \Y$ which is unknown to the robot; here, $\theta \in \mathbb{R}^{n_p}$ are external parameters (e.g., location of obstacles, map of environment, etc.). The observations may be corrupted by (possibly stochastic) sensor noise $\dy(t) \in \mathbb{R}^{n_{w_y}}$, where $\Vert \dy(t)\Vert \le \bardy$, for all $t$. We note that our results also apply to time-varying $B(t)$ under some conditions on its null-space.

We assume that \eqref{eq:system} is locally incrementally exponentially stabilizable (IES) in domain $D_c \subseteq \X$, that is, there exists an $\alpha$, $\lambda > 0$, and some feedback controller such that for any nominal trajectory $x^*(t) \subseteq D_c$, $\Vert x^*(t) - x(t)\Vert \le \alpha e^{-\lambda t} \Vert x^*(0) - x(0)\Vert$ for all $t$. While stronger than asymptotic stability, many underactuated systems are IES \cite{DBLP:conf/isrr/ManchesterTS15}. We also assume that \eqref{eq:sys_and_obs} is locally universally detectable \cite{DBLP:conf/aucc/ManchesterS14}, which ensures that any two trajectories $x_1(t)$ and $x_2(t)$ in a domain $D_e \subseteq \X$ that yield identical observations $y_1(t)$ and $y_2(t)$ for all $t$ converge to each other as $t \rightarrow \infty$, i.e., $x_1(t) \rightarrow x_2(t)$. Similar assumptions are common in the estimation literature \cite{Maybeck79stochasticsmodels} to ensure estimator convergence, and do not require the full state to be observable instantaneously, e.g., as in \cite{dean}.

\noindent \textbf{Definitions}: We assume $\X$ is partitioned into (un)safe ($\Xunsafe$) $\Xsafe$ sets (e.g., obstacles). Let ($\mathbb{S}_n^{>0}$) $\mathbb{S}_n$ be the set of (positive definite) symmetric $n \times n$ matrices. For $Q \in \mathbb{S}_n$, denote $\bar\lambda(Q)$ and $\underline\lambda(Q)$ as its maximum and minimum eigenvalues. If $Q(x)$ is a matrix-valued function over a domain $D$, we denote $\bar\lambda_D(Q) \doteq \sup_{x\in D} \bar\lambda(Q(x))$ and $\underline\lambda_D(Q) \doteq \inf_{x\in D} \underline\lambda(Q(x))$. Let the Lie derivative of a matrix-valued function $Q(x) \in \mathbb{R}^{n\times n}$ along a vector $y \in \mathbb{R}^n$ be denoted as $\partial_y Q(x) \doteq \sum_{i=1}^n y^i \frac{\partial Q}{\partial x^i}$, where $x^i$ is the $i$th element of vector $x$.
For a smooth manifold $\X$, a Riemannian metric tensor $M: \X \rightarrow \mathbb{S}_{n_x}^{>0}$ provides the tangent space $T_x \X$ with an inner product $\delta_x^\top M(x) \delta_x$, where $\delta_x \in T_x \X$. The length $l(c)$ of a curve $c: [0,1]\rightarrow \X$ between $c(0)$, $c(1)$ is $l(c) \doteq \int_0^1 \sqrt{V(c(s), c_s(s))} ds$, where $V(c(s), c_s(s)) \doteq c_s(s)^\top M(c(s))c_s(s)$, and $c_s(s) \doteq \partial c(s)/\partial s$. The Riemannian distance between $p,q\in\X$ is $d(p,q) \doteq \inf_{c\in \mathcal{C}(p,q)} l(c)$, where $\mathcal{C}(p,q)$ contains all smooth curves between $p$ and $q$; a curve $\gamma(p, q)$ achieving the argmin is called a geodesic.

\subsection{Problem statement}\label{sec:prob}

We formally state the output feedback motion planning problem (OFMP) as follows:

\noindent \textbf{OFMP}: Given start $x_I$, external parameter $\theta \in D_\theta$, goal region $\mathcal{G} \subseteq D_x$ ($D_\theta$, $D_x$ are defined in the next paragraph), and safe set $\normalfont\Xsafe$, we want to plan a state-control trajectory $x^*: [0, T] \rightarrow \X$, $u^*: [0, T] \rightarrow \U$, $x^*(0) = x_I$, under the nominal dynamics $\dot x(t) = f(x(t)) + Bu(t)$ such that in execution on the true system \eqref{eq:system}, $x(t) \in \Xsafe$ for all $t \in [0,T]$ and $x(T) \in \mathcal{G}$. At runtime, we do not observe $x(t)$; we are only given observations $y(t)$ generated by \eqref{eq:observation}, and must track $x^*$ using a (dynamic) output feedback controller that we must also design. We assume $f$, $B$, $B_w$, and $B_y$ are known; $h$ is unknown; $\dx$, $\dy$ are not measurable but $\bardx$ and $\bardy$ are known. If $n_r \le n_x$ of the states can be inferred directly from $y$, we denote these indices as the reduced observation $\yred = \Cred x \in \mathbb{R}^{n_r}$, where $\Cred \in \{0,1\}^{n_r \times n_x}$ is a boolean matrix that selects the observable dimensions of $x$. We assume that we are given $\Cred$. Let $x(t)$ be the executed trajectory of \eqref{eq:system}, and let $\xhat(t)$ be the trajectory of the state estimate. We are given upper bounds $\bar d_c(0)$, $\bar d_e(0)$on the Riemannian distance between the true and estimated initial state $\diste(x(0), \xhat(0))$ and between the true/planned initial state $\distc(x^*(0), x(0))$; $\diste(\cdot,\cdot)$ and $\diste(\cdot,\cdot)$ are defined with respect to (w.r.t.) metrics $\Mc$ and $\Mo$, defined in Sec. \ref{sec:prelim_ccm}. 

To help solve the OFMP, we are given two datasets. The first is $\S = \{h(x_i, \theta_i), x_i, \theta_i\}_{i=1}^{\Ndata}$, a dataset of noiseless (cf. Sec. \ref{sec:conclusion} for discussion on how to relax this assumption) observation-state-parameter triplets, where $x_i \in D_p \subseteq \X$, $\theta_i \in D_\theta \subseteq \Theta$ are collected by any means (sampling, demonstrations, etc.). We assume $D_p$ and $D_\theta$ (the domains where $\S$ is drawn from) are known, though this can be relaxed by estimating these sets as in \cite{lipschitz_ral, cdc}. We are also given a validation dataset $\V = \{h(x_i, \theta_i), x_i, \theta_i\}_{i=1}^{N_\textrm{val}}$ collected i.i.d. in $D_p \times D_\theta$. In the context of \eqref{eq:observation}, $h(x,\theta)$ may be a simulated image, and $B_y\dy(t)$ is the sensor noise at runtime. We also define a ``trusted domain'' for planning, $D = D_x \times D_\theta \subseteq \X \times \Theta$, where $D_x = D_r \cap D_c \cap D_e$ and $D_r$ is defined as follows: for ease, suppose $\Cred$ selects the first $n_r$ indices of $x$, then $D_r = (\Cred D_p) \times \mathbb{R}^{n_x - n_r}$. $D_r$ is defined similarly if $\Cred$ selects other indices (cf. Fig. \ref{fig:dr}). Ultimately, $D_x$ is a set where a stabilizing controller (in $D_c$) and state estimator (in $D_e$) exist, and where the perception is valid (in $D_r$).

\subsection{Control/observer contraction metrics (CCMs/OCMs)}\label{sec:prelim_ccm}

As our approach builds on contraction theory, we provide an overview here. Control contraction theory \cite{manchester} studies incremental stabilizability by measuring the distances between trajectories w.r.t. a Riemannian metric $\Mc: \X \rightarrow \mathbb{S}_{n_x}^{>0}$. For \eqref{eq:system} if $\dx \equiv 0$, a sufficient condition \cite{sumeet_icra} for $\Mc$ to be called a control contraction metric (CCM) is:

\vspace{-8pt}
\begin{subequations}\label{eq:con_ccm}
\begin{eqnarray}
	\gone_{\perp}^\top \Big(-\partial_{\gzero} \Wc(x) + A(x) \Wc(x) + \Wc(x)A(x)^\top + 2 \lambda_c \Wc(x)\Big) \gone_{\perp} \preceq 0 \label{eq:strong_con}\\[-8pt]
	\gone_{\perp}^\top \Big( \partial_{\gone^j} \Wc(x) \Big) \gone_{\perp} = 0,\quad j = 1...n_u,\label{eq:strong_orth}
\end{eqnarray}
\end{subequations}
\vspace{-8pt}

\noindent for all $x \in D_c$, where $\Wc(x) \doteq \Mc^{-1}(x)$, $A(x) = \frac{\partial \gzero(x)}{\partial x}$, and $B_\perp$ is a basis for the null-space of $B$. The CCM also defines a controller $u: \X\times\X\times\U \rightarrow \U$, which takes the current state $x(t)$ and a state/control $x^*(t)$, $u^*(t)$ on the nominal state/control trajectory being tracked $x^*: [0, T] \rightarrow \X$, $u^*: [0, T] \rightarrow \U$, and returns a $u$ that contracts $x$ towards $x^*$ at rate $\lambda_c > 0$. The controller $u(x, x^*, u^*)$ can be computed directly via $\Wc(x)$ (cf. Sec. \ref{sec:method_bounds}). If $\dx \equiv 0$, for any nominal $x^*(t)$, applying $u(x,x^*,u^*)$ renders the system closed-loop IES, i.e., $\Vert x(t) - x^*(t) \Vert \le \alpha_c \Vert x(0) - x^*(0)\Vert e^{-\lambda_c t}$ for $\alpha_c > 0$. For bounded $\dx$, \eqref{eq:system} remains in a tube around $x^*(t)$; we exploit this in Sec. \ref{sec:method_bounds}. Contraction also analyzes the convergence of state observers \cite{DBLP:conf/aucc/ManchesterS14, DBLP:journals/tac/DaniCH15}, i.e., whether a state estimate $\xhat(t)$ approaches the true state $x(t)$. Consider the nominal closed-loop system $\dot x = f(x) + Bu(\xhat,x^*,u^*)$ with noiseless observations $y = h(x,\theta)$ and a nominal observer

\vspace{-8pt}
\begin{equation}\label{eq:observer}
    \dot{\hat x} = f(\hat x) + Bu(\xhat,x^*,u^*) + \textstyle\frac{1}{2}\rho(\xhat) \Mo(\hat x) C(\hat x)^\top (y - h(\hat x,\theta))
\end{equation}

\noindent for the nominal system, where $C(x) = \frac{\partial h(x,\theta)}{\partial x}$, $\rho(x) \ge 0$ is a  multiplier term, and $\Mo: \X \rightarrow \mathbb{S}_{n_x}^{>0}$ is called an observer contraction metric (OCM), which should satisfy

\vspace{-10pt}
\begin{equation}\label{eq:con_ocm}
	\partial_{f+Bu} \Wo(\xhat) + \Wo(\xhat)A(\xhat) + A(\xhat)^\top \Wo(\xhat) - \rho(\xhat)C(\xhat)^\top C(\xhat) \le -2 \lambda_e \Wo(\xhat)
\end{equation}

\noindent for all $\xhat \in D_e \subseteq \X, u \in \U$. Here, $\Wo(\xhat) = \Mo^{-1}(\xhat)$. To show that the estimated and true trajectories $\xhat(t)$ and $x(t)$ converge, we can analyze a nominal ``meta-level'' virtual system with state $q$ \cite{DBLP:journals/csysl/TsukamotoC21}, which recovers the nominal $x(t)$ and $\hat x(t)$ when integrated from initial conditions $q(0) = x(0)$ and $q(0) = \xhat(0)$: 

\begin{equation}\label{eq:system_virtual_nocontrol}
	\dot q = f(q) + Bu(\xhat,x^*,u^*) + \textstyle\frac{1}{2}\rho(\xhat) \Mo(\hat x) C(\hat x)^\top (y - h(q,\theta)).
\end{equation}

By setting $q = \xhat$, we recover the estimator dynamics \eqref{eq:observer}; if we set $q = x$, we recover $\dot x = f(x) + Bu(\xhat,x^*,u^*)$. We can then analyze the convergence of $\xhat$ to $x$ via \eqref{eq:system_virtual_nocontrol}, and \cite{DBLP:journals/csysl/TsukamotoC21} shows that if \eqref{eq:con_ocm} holds, then $\xhat(t)$ contracts at some rate $\gamma \in (0, \lambda_e]$ towards $x(t)$. If $\Mo(x)$ and $C(x)$ are constant, one can show that this holds for $\gamma = \lambda_e$. In particular, $\Vert x(t) - \xhat(t) \Vert \le \alpha_e \Vert x(0) - \xhat(0)\Vert e^{-\lambda_e t}$ for $\alpha_e > 0$, and $\xhat(t)$ remains in a tube around $x(t)$ if \eqref{eq:observer} is perturbed. For polynomial systems of moderate dimension ($n_x \lesssim 12$) with polynomial observation maps, CCMs and OCMs can be found via convex Sum of Squares (SoS) programs \cite{sumeet_icra}. CCMs/OCMs can also be found for high-dimensional non-polynomial systems via learning-based methods (e.g., \cite{dawei, cdc}).  

\section{Method}

\begin{figure}
\includegraphics[width=\textwidth]{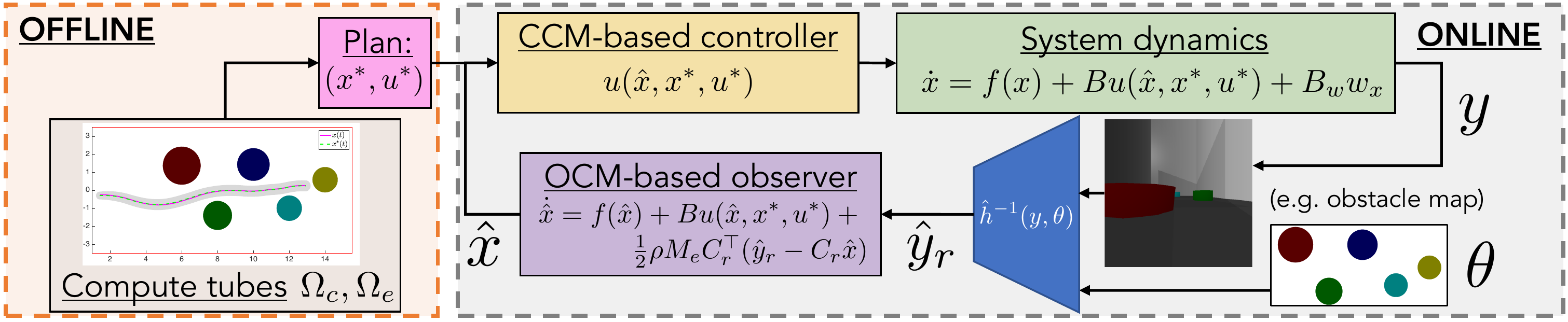}
\caption{Our method. \textbf{Offline}: After learning a perception system $\hinv$ (Sec. \ref{sec:method_learning}), we bound its error to derive tracking tubes under imperfect perception (Sec. \ref{sec:method_bounds}). We use these tubes to find safely-trackable plans (Sec. \ref{sec:method_ofmp}). \textbf{Online}: We design a CCM/OCM-based controller/observer (Sec. \ref{sec:method_ccms_and_ocms}) to track the plan/perform state estimation at runtime, using $\hinv$ to process rich observations $y$. } \label{fig:method}
\end{figure}

We describe our solution to the OFMP (cf. Fig. \ref{fig:method}). Using dataset $\S$, we first train a perception system that returns a reduced-order observation that simplifies the search for the contraction metrics (Sec. \ref{sec:method_learning}). Second, we bound the error of the learned perception module, and propagate this perception error bound through the system to derive bounds on the tracking and estimation error when using a CCM-/OCM-based controller/estimator (Sec. \ref{sec:method_bounds}). Third, we obtain a CCM and OCM which optimizes this bound via SoS programming (Sec. \ref{sec:method_ccms_and_ocms}). Finally, we use these bounds to constrain a planner to return trajectories that enable safe runtime tracking and robust goal reachability from observations (Sec. \ref{sec:method_ofmp}). For space, all proofs for the theoretical results are in App. \ref{sec:app_proofs}.

\subsection{Learning a perception module for contraction-based estimation}\label{sec:method_learning}

Let us reconsider the observer \eqref{eq:observer}, which updates its estimate directly using $y - h(\xhat,\theta)$ in the rich observation space. To implement \eqref{eq:observer}, one can use $\S$ to train a deep approximation of $h$, denoted $\hat h$, design an OCM satisfying \eqref{eq:con_ocm} for $C(\xhat) = \frac{\partial \hat h(\xhat,\theta)}{\partial x}$, and plug $\hat h$ and the OCM into \eqref{eq:observer}. This na\"ive solution is flawed: 1) as $n_y$ is large, learning an accurate $\hat h$ can be difficult; 2) the $C(\hat x)$ in \eqref{eq:con_ocm} becomes the Jacobian of a (non-polynomial) deep network, complicating OCM synthesis by precluding the use of SoS programming. 

We can take a more structured approach if we know which states can be directly inferred from $y$; this is reasonable if the states have semantic meaning (e.g., poses, velocities). Recall $\Cred$ (Sec. \ref{sec:prob}) defines this reduced observation as $\yred = \Cred x \in \mathbb{R}^{n_r}$. We can then learn an approximate inverse $\hinv(y,\theta): \mathbb{R}^{n_y}\times \mathbb{R}^{n_p} \rightarrow \mathbb{R}^{n_r}$ which maps a $y$ and $\theta$ to the reduced observation. Note that if each unique $y$ corresponds to a unique $\yred$, this inverse is well-defined and does not require the full state to be invertible from a single $y$. Concretely, consider a car with position, orientation, and velocity states $[p_x, p_y, \phi, v]$ and RGB-D data from an onboard camera (Fig. \ref{fig:environments}.A) driving in several obstacle fields. In this case, $y_r = [p_x, p_y, \phi]^\top$ and $\theta$ could be the obstacle locations. We model $\hinv$ as a neural network and train it via the mean squared error between $\hinv(y_i, \theta_i)$ and $\Cred x_i$ for all $i \in 1,\ldots,\Ndata$. Note that as the nominal reduced observations are roughly linear, i.e., $\hinv(h(x,\theta),\theta) \approx \Cred x$, this simplifies the nominal observer \eqref{eq:observer} to
    $\dot{\xhat} = f(\xhat) + Bu(\hat x, x^*, u^*) + \textstyle\frac{1}{2}\rho(\xhat)\Mo(\xhat) \Cred^\top\Cred (x - \xhat)$,
and simplifies OCM synthesis: as $\Cred^\top \Cred$ is constant, \eqref{eq:con_ocm} is SoS-representable, despite $\hinv$ being non-polynomial. Compared to the nominal reduced observer, the true observer we use,

\vspace{-13pt}
\begin{equation}\label{eq:observer_reduced}
    \dot{\xhat} = f(\xhat) + Bu(\hat x, x^*, u^*) + \textstyle\frac{1}{2}\rho(\xhat)\Mo(\xhat) \Cred^\top (\hinv(h(x,\theta)+B_y\dy,\theta) - \Cred \xhat),
\end{equation}

\noindent experiences disturbance from model error $B_w \dx$, sensor noise $B_y \dy$, and learning error $\Vert \hinv(h(x,\theta),\theta) - \Cred x\Vert$. Quantifying these errors for our vision-based observer \eqref{eq:observer_reduced} is one of our core contributions and is key in deriving tracking bounds useful for planning.

\subsection{Bounding tracking error and state estimation error for planning}\label{sec:method_bounds}

To begin, assume we have a CCM $\Mc$ and an OCM $\Mo$ that are valid in $D_c \subseteq \X$ and $D_e \subseteq \X$ and which contract at rate $\lambda_c$ and $\lambda_e$, respectively. We discuss CCM/OCM synthesis in Sec. \ref{sec:method_ccms_and_ocms}. Define the nominal closed-loop state and virtual dynamics as:

\vspace{-10pt}
\begin{subequations}\label{eq:nominal}\small
	\begin{align}
     \dot x(t) &= f(x(t)) + Bu(x(t), x^*(t), u^*(t))\label{eq:nominal_x}\\
     \dot q(t) &= f(q(t)) + Bu(\xhat(t), x^*(t), u^*(t)) + \textstyle\frac{1}{2}\rho(q(t))\Mo(q(t))\Cred^\top\Cred(x(t) - q(t))\label{eq:nominal_xhat}
\end{align} 
\end{subequations}
\vspace{-10pt}

\noindent Factor the CCM/dual OCM as $\Mc(x) = R_c(x)^\top R_c(x)$ and $\Wo(x) = R_e(x)^\top R_e(x)$. Let $\gamma_c^t(s)$, $s\in[0,1]$ be the geodesic between $x^*(t)$ and $x(t)$ w.r.t. $\Mc$, and $\gamma_e^t(s)$, $s\in[0,1]$ be the geodesic between $\xhat(t)$ and $x(t)$ w.r.t. $\Wo$. \cite{DBLP:conf/aucc/ManchesterS14} shows if $\gamma_c^t(s) \subseteq D_c$ for all $t,s$ and \eqref{eq:nominal_x} is perturbed by $\dc(t)$, i.e., $\dot x = f(x) + Bu(x, x^*, u^*) + \dc$, the Riemannian distance w.r.t. $\Mc$ between the true and nominal state, $\distc(t) = \distc(x^*(t), x(t))$, satisfies:

\vspace{-2pt}
\begin{equation}\label{eq:rc}
	\dot \distc(t) \le -\lambda_c \distc(t) + \textstyle\int_0^1 \Vert R_c(\gamma_c^t(s)) \dc(t) \Vert ds.
\end{equation}
\vspace{-7pt}

\noindent If $\gamma_e^t(s) \subseteq D_e$ for all $t,s$ and \eqref{eq:nominal_xhat} is perturbed by additive $\dq(t)$  \cite{DBLP:journals/csysl/TsukamotoC21}, the Riemannian distance w.r.t. $\Wo$ between the true and estimated state, $\diste(t) = \diste(x(t), \xhat(t))$, satisfies

\vspace{-5pt}
\begin{equation}\label{eq:re}
	\dot \diste(t) \le -\lambda_e \diste(t) + \textstyle\int_0^1 \Vert R_e(\gamma_e^t(s)) \dq(t) \Vert ds.
\end{equation}
\vspace{-11pt}

\noindent We will use \eqref{eq:rc} and \eqref{eq:re} to obtain upper bounds on the tracking/estimation Riemannian distances, denoted as $\bar\distc(t)$ and $\bar\diste(t)$, respectively. These upper bounds define tracking and state estimation tubes, i.e., a bound on where $x$ and $\xhat$ can be, which we denote as $\Omega_c(t) = \{x \mid \distc(x^*(t), x) \le \bar\distc(t)\}$ and $\Omega_e(t) = \{\xhat \mid \diste(x(t), \xhat) \le \bar\diste(t)\}$, respectively. These tubes are crucial in informing where the planner can safely visit, since tracking any $\Omega_c$-buffered candidate trajectory within $D_x$ which remains in $\Xsafe$ is guaranteed to remain safe. However, for these tubes to be usable in a planner, we need explicit bounds on the integral terms in \eqref{eq:rc} and \eqref{eq:re}.

In this section, we first present the final derived bounds on the integrals (Lemmas \ref{lem:dc} and \ref{lem:de}), describe the ideas behind the derivations, and postpone the full mathematical details to App. \ref{sec:app_bounds}.

\begin{lemma}[$\dot d_c(t)$]\label{lem:dc}
The integral term in \eqref{eq:rc} can be bounded as
	\begin{equation}\label{eq:dc_bound}
		\textstyle\int_0^1 \Vert R_c(\gamma_c^t(s)) \dc(t) \Vert ds \le \sqrt{\maxeigc{\Mc}}\bardx + L_{\Delta k} \diste.
	\end{equation}
\end{lemma}
\noindent In the second term, $L_{\Delta k}$ is the Lipschitz constant of the controller error (to be described later) which, together with state estimate error $\diste$, bounds the destabilizing effect of using incorrect state estimates in feedback control. This term, which can be explicitly estimated and thus concretely informs tube size in planning, is the key novelty of Lemma \ref{lem:dc}. Overall, \eqref{eq:dc_bound} states that tracking degrades with larger dynamics and estimation error.

\begin{lemma}[$\dot d_e(t)$]\label{lem:de}
Let $\bar\sigma(B_y)$ denote the maximum singular value of $B_y$. For constant $\rho$ and $\Mo$, the integral in \eqref{eq:re} simplifies to $\Vert R_e \dq(t) \Vert$ and can be bounded as:
	\begin{equation}\label{eq:lem2}\small
		\Vert R_e \dq(t) \Vert \le \sqrt{\maxeig{\Wo}}\bardx + \textstyle\frac{1}{2}\rho\maxeig{\Mo}^{1/2}\big(L_{\hinv}\sqrt{\maxsing{B_y}}\bardy + \bar\err_{\{1,2,3\}}(x^*,\theta)\big)
	\end{equation}
\end{lemma}
\noindent We write Lemma \ref{lem:de} for constant $\rho$ and $\Mo$, as this is the representation used in Sec. \ref{sec:results}. Here, $L_{\hinv}$ is the local Lipschitz constant of $\hinv$, and $\bar\epsilon_{\{1,2,3\}}(x^*,\theta)$ are (spatially-varying) bounds on its error $\Vert \hinv(h(x,\theta),\theta) - \Cred x\Vert$, each with different strengths/weaknesses (cf. Fig. \ref{fig:bounds} for a visual overview). Relative to prior work, Lemma \ref{lem:de} is novel as it bounds high-dimensional measurement error and learned perception module error. Overall, \eqref{eq:lem2} states that estimation accuracy degrades with larger dynamics error, measurement error, and learned perception module error.

\vspace{-5pt}
\subsubsection{Bounding tracking error:}\label{sec:bound_trk_err}
We explain more details behind Lemma \ref{lem:dc}. As Lemma \ref{lem:dc} relies on a bound for $\dc(t)$, we first break down the components that make up $\dc(t)$. Relative to the nominal closed-loop dynamics \eqref{eq:nominal_x}, our true closed-loop system
\begin{equation}\label{eq:system_closed_loop}
	\dot x(t) = f(x(t)) + Bu(\xhat(t), x^*(t), u^*(t)) + B_w(t) \dx(t)
\end{equation}
 is subject to two disturbances. The first is the dynamics error $B_w(t) \dx(t)$. The second is imperfect state feedback: we apply $u(\xhat, x^*, u^*)$ instead of $u(x, x^*, u^*)$, which unlike the latter, may not stabilize \eqref{eq:nominal_x} at rate $\lambda_c$. Na\"ively, one can bound this error by rewriting \eqref{eq:system_closed_loop} as $\dot x = f(x) + \textcolor{Red}{Bu(\xhat, x^*, u^*) - Bu(x, x^*, u^*)} + Bu(x, x^*, u^*) + B_w \dx$,
where the difference between output/perfect state feedback is in \textcolor{Red}{red}. While $\Vert Bu(\xhat, x^*, u^*) - Bu(x, x^*, u^*)\Vert$ is a valid disturbance bound, we can obtain a tighter bound by exploiting the structure of $u(x, x^*, u^*)$. In general, many $u_\textrm{fb} \in \U$ can make $\dot x = f(x) + B(u^*+u_\textrm{fb})$ contract at rate $\lambda_c$ towards $x^*$, w.r.t. $\Mc$. Define $\dot x^* = f(x^*) + Bu^*$; then, the contracting $u_\textrm{fb}$ \cite{sumeet_icra} are defined by a linear inequality constraint, 

\vspace{-7pt}
\begin{equation}\label{eq:ufeas}\small
	\U_\textrm{feas}(x,x^*,u^*) = \{u_\textrm{fb} \mid \gamma_{c,s}^\top(1) \Mc(x){\dot x} - \gamma_{c,s}^\top(0)\Mc(x^*)\dot x^* \le -\lambda_c \distc(x^*, x)^2\},
\end{equation}

\setlength{\intextsep}{5pt}%
\setlength{\columnsep}{10pt}%
\begin{figure}
    \centering
    \includegraphics[width=0.5\textwidth]{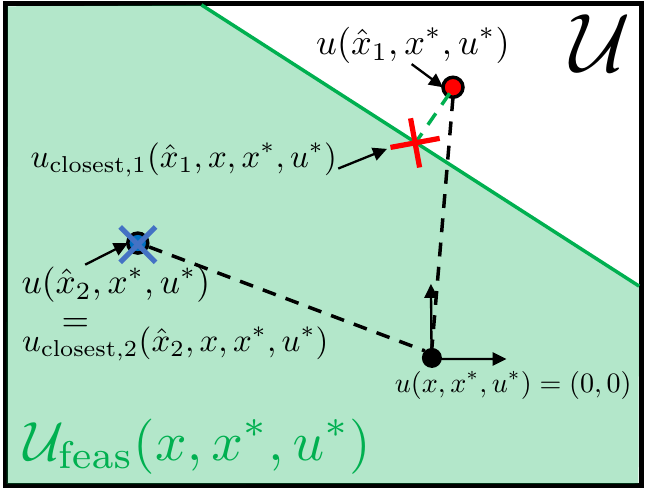}
    \caption{$u_\textrm{closest}$ can be much closer to $u(\xhat,x^*,u^*)$ than $u(x,x^*,u^*)$: we show this for two different state estimates $\xhat_1$ and $\xhat_2$.}
    \label{fig:ufeas}
    \end{figure}

\noindent where $\gamma_{c,s}(\cdot) = \frac{\partial\gamma_c(\cdot)}{\partial s}$.
As in \cite{sumeet_icra}, we select the minimum-norm feasible control to be $u(\xhat, x^*, u^*)$, i.e., $u(\xhat, x^*, u^*) = \arg\min_{u \in \U_\textrm{feas}(\xhat,x^*,u^*)} \Vert u \Vert$. Then, using $\U_\textrm{feas}$, we can rewrite \eqref{eq:system_closed_loop} as $\dot x = f(x) + \textcolor{Red}{B(u(\xhat, x^*, u^*) - u_\textrm{closest}} + u_\textrm{closest}) + B_w \dx$,
where $u_\textrm{closest}(\xhat, x, x^*, u^*) \doteq \arg\min_{u \in \U_\textrm{feas}(x,x^*,u^*)} \Vert u - u(\xhat, x^*, u^*)\Vert$ is the closest control input to $u(\xhat, x^*, u^*)$ that contracts the nominal dynamics at $x$. Bounding the imperfect state feedback as $\Vert Bu(\xhat, x^*, u^*) - Bu_\textrm{closest}\Vert $ instead of $\Vert Bu(\xhat, x^*, u^*) - Bu(x, x^*, u^*)\Vert$ can be far tighter, as $u(\xhat,x^*,u^*)$ may still contract the system at rate $\lambda_c$ (Fig. \ref{fig:ufeas}: $\xhat_2$ case), or there can be a contracting $u$ closer to $u(\xhat,x^*,u^*)$ than $u(x,x^*,u^*)$ (Fig. \ref{fig:ufeas}: $\xhat_1$ case). Combining with the dynamics error, we can write $\dc$:
\begin{equation}\label{eq:dc}
	\dc(t) \doteq Bu(\xhat(t), x^*(t), u^*(t)) - Bu_\textrm{closest}(t) + B_w(t) \dx(t)
\end{equation}
As \eqref{eq:dc} still depends on $x$ and $\xhat$, which are unknown at planning time, extra steps must be taken to obtain a useful bound that is independent of $x$ and $\xhat$; we achieve this by bounding the first two terms of \eqref{eq:dc} via a Lipschitz constant. Define $\Delta k(\xhat, x, x^*, u^*) = \textstyle\max_{s \in [0,1]} \Vert R_c(\gamma_c^t(s)) B(u(\xhat, x^*, u^*) - u_\textrm{closest})\Vert$,
and $L_{\Delta k}$ as its local Lipschitz constant in the first argument, i.e., for all $x^* \in D$, $u^* \in \U$, $\{x \mid \distc(x^*,x) \le \bar c\}$, and $\{\xhat \mid \diste(x,\xhat) \le \bar e\}$ for predetermined $\bar c, \bar e > 0$ (adjustable based on the expected error),

\begin{equation}\label{eq:lip}
	|\Delta k(\xhat_1, x, x^*, u^*) - \Delta k(\xhat_2, x, x^*, u^*))| \le L_{\Delta k}\diste(\xhat_1, \xhat_2).
\end{equation}
See Rem. \ref{rem:lipschitz_estimation} for details on estimating $L_{\Delta k}$. In estimating $L_{\Delta k}$, we measure input distances w.r.t. $\Wo$; this reduces conservativeness due to the form of our estimation error bound. Combining \eqref{eq:dc}-\eqref{eq:lip} yields Lemma \ref{lem:dc}; see App. \ref{sec:app_proofs} for the detailed proof.

\subsubsection{Bounding estimation error:}
Now, we provide more details behind Lemma \ref{lem:de}. To bound $\textstyle\int_0^1 \Vert R_e(\gamma_e^t(s)) \dq(t) \Vert ds$, we first note that $\Vert\dq\Vert$ is bounded by the sum of the disturbance magnitudes when $q = x$ and when $q = \xhat$ \cite{DBLP:journals/csysl/TsukamotoC21}. If $q = x$, \eqref{eq:nominal_xhat} becomes $\dot x = f(x) + Bu(\xhat,x^*,u^*)$; relative to the true closed-loop dynamics \eqref{eq:system_closed_loop}, the disturbance is $B_w(t) \dx(t)$.
If instead $q = \xhat$, \eqref{eq:nominal_xhat} becomes $\xhatdot = f(\xhat) + Bu(\xhat, x^*, u^*) + \textstyle\frac{1}{2}\rho(\xhat)\Mo(\xhat)\Cred^\top \Cred(x - \xhat)$; relative to the true observer \eqref{eq:observer_reduced}, the disturbance is
\begin{equation}\label{eq:de}
\de(t) \doteq \textstyle\frac{1}{2}\rho(\xhat(t))\Mo(\xhat(t))\Cred^\top\textcolor{Red}{(\hinv\big(\h(x(t), \theta) + B_y \dy(t), \theta\big) - \Cred x(t))}.
\end{equation}
Two errors drive $\de(t)$: the perception error $\hinv(h(x, \theta),\theta) - \Cred x$, and the runtime observation noise $B_y\dy$. Combining with the dynamics error gives $\dq(t) \doteq B_w(t) \dx(t) + \de(t)$.
$B_w(t) \dx(t)$ can be bounded as in Lemma \ref{lem:dc}, but $\de(t)$ is harder to bound. Let $y_p = h(x,\theta)$ and $y = h(x,\theta)+B_y\dy$. We rewrite the norm of the \textcolor{Red}{red} term in \eqref{eq:de} as:
\begin{equation}\label{eq:dq_derivation}
	\begin{array}{ll}
		\Vert \hinv\big(y, \theta\big) - \Cred x \Vert &= \Vert \hinv\big(y, \theta\big) - \hinv\big(y_p, \theta\big) + \hinv\big(y_p, \theta\big) - \Cred x \Vert \\
		&\le \underbrace{L_{\hinv}\Vert B_y \dy\Vert}_{\textrm{from measurement noise}} + \underbrace{\Vert \hinv\big(y_p, \theta\big) - \Cred x \Vert}_{\textrm{from learning error} \doteq \err(x, \theta) }.
	\end{array}
\end{equation}
Here, $L_{\hinv}$ is the local Lipschitz constant of the learned inverse function in $y$, i.e.,
\begin{equation}\label{eq:lip_hinv}
	\Vert \hinv(\tilde y, \theta) - \hinv(\check y, \theta) \Vert \le L_{\hinv}\Vert \tilde y - \check y\Vert,\quad \forall \tilde y, \check y \in D_y \oplus \Y_d,\quad \forall \theta \in D_\theta,
\end{equation}
where $D_y = h(D_r,D_\theta)$ is the image of the training data domains, $\oplus$ is the Minkowski sum, and $\Y_d = \{B_y \dy \mid \Vert \dy \Vert \le \bardy\}$ is the set of feasible measurement noise. The first braced term in \eqref{eq:dq_derivation} bounds the effect of measurement error on the reduced observation and is valid for all $(x, \theta) \in D_r \times D_\theta$ and observation noise satisfying $\Vert \dy \Vert \le \bardy$. 

\vspace{5pt}
\begin{figure}
\includegraphics[width=\textwidth]{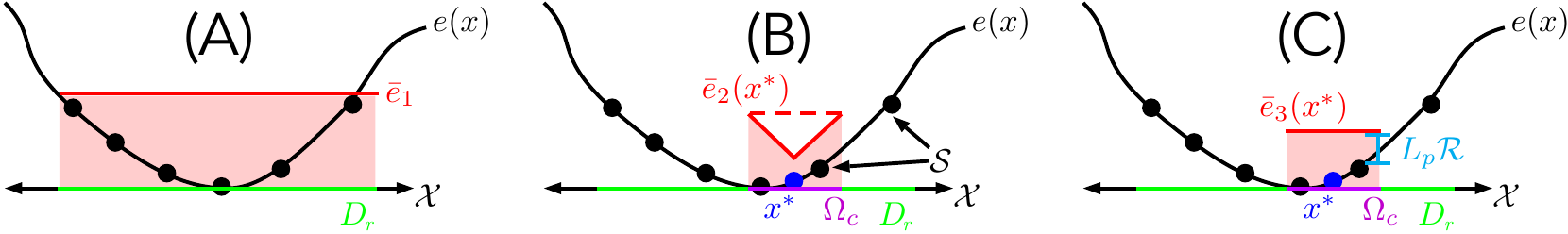}
\caption{Our perception error bounds. (A) $\bar\err_1$ is simple but conservative. B) $\bar\err_2(x^*)$ is tighter, as it only seeks to be valid over the tube $\Omega_c$. However, it scales linearly with the size of $\Omega_c$. C) $\bar\err_3(x^*)$ can be tighter for larger $\Omega_c$ by adding a Lipschitz-based buffer to the largest training error in $\Omega_c$.} \label{fig:bounds}
\end{figure}\vspace{5pt}

Now, consider the second braced term in \eqref{eq:dq_derivation}. How can we bound the learned perception module error $\err(x,\theta) \doteq \Vert \hinv\big(\h(x, \theta), \theta\big) - \Cred x \Vert$ over $D_r \times D_\theta$? We describe three options (Fig. \ref{fig:bounds}) at a high level, highlight their strengths/drawbacks, and provide the details in App. \ref{sec:app_bounds}. The first bound, denoted $\bar\err_1$, is a constant bound on $\err(x,\theta)$ globally over $D_r \times D_\theta$ (Fig. \ref{fig:bounds}.A). This works well if the error is consistent, but is loose if there are any error spikes. The second bound (Fig. \ref{fig:bounds}.B), denoted $\bar\err_2(x^*,\theta)$, bounds the error only in the tube $\Omega_c$ around a nominal $x^*$, using the Lipschitz constant of $\err(x,\theta)$ (denoted $L_p$). Due to its locality, $\bar\err_2(x^*,\theta)$ can be tighter than $\bar\err_1$; however, it scales linearly with the size of $\Omega_c$, even if $\err(x,\theta)$ remains constant. The third bound, $\bar\err_3(x^*,\theta)$ (Fig. \ref{fig:bounds}.C), also bounds the error in the tube but avoids the linear scaling by taking the worst training error in $\Omega_c$ and buffering it with a \textit{constant} value, which depends on $L_p$ and the dataset dispersion $\disp$. Each of these bounds $\bar\err_{\{1,2,3\}}$ on $\err(x,\theta)$ can be plugged into Lemma \ref{lem:de} to upper bound $\err(x,\theta)$; see App. \ref{sec:app_bounds} for details.

\subsubsection{Integrating the differential inequalities:}
Now that we can bound the RHSs of the differential inequalities \eqref{eq:rc} and \eqref{eq:re} via Lemmas \ref{lem:dc} and \ref{lem:de}, we show how these bounds on $\dot\distc$ and $\dot\diste$ bound the values of $\distc$ and $\diste$, thereby providing the desired tubes. By grouping terms in \eqref{eq:rc}-\eqref{eq:re}, we have the following affine vector-valued differential inequality, 
\begin{equation}\label{eq:vector_inequality}\small
	\hspace{-1pt}\begin{bmatrix}
		\dot \distc \\
		\dot \diste
	\end{bmatrix} \le \begin{bmatrix}
		-\lambda_c & L_{\Delta k} \\
		(*) & -\lambda_e
	\end{bmatrix} \begin{bmatrix}
 \distc \\ \diste	
 \end{bmatrix} + \begin{bmatrix}
 	\sqrt{\maxeigc{\Mc}} \bardx \\ \sqrt{\maxeig{\Wo}} \bardx + \textstyle\frac{\rho}{2}\maxeig{\Mo}^{1/2}\big(L_{\hinv}\bardy + \bar\err_{\{1,\tilde 2,3\}}(x^*,\theta)\big)
 \end{bmatrix},\hspace{-15pt}
\end{equation}
where we regroup the terms for $\bar\err_2(x^*,\theta)$ as  $\bar\err_{\tilde 2}(x^*, \theta) \doteq \bar\err_{2}(x^*, \theta)- L_p \distc/\sqrt{\mineigc{\Mc}}$, and $(*) = 0.5 L_p \rho\sqrt{\maxeig{\Mo}/\mineigc{\Mc}}$ if using $\bar\err_2$ and 0 else. Then, we have this result:
\begin{theorem}[From derivative to value]\label{thm:upper_bound}
Let RHS denote the right hand side of \eqref{eq:vector_inequality}. Given bounds on the Riemannian distances at $t=0$: $\distc(0) \le \bar d_c(0)$ and $\diste(0) \le \bar d_e(0)$, upper bounds $\bar\distc(t) \ge \distc(t)$ and $\bar\diste(t) \ge \diste(t)$ for all $t \in [0, T]$ can be written as 
	\begin{equation}\label{eq:trk_bnd}\small
		\begin{bmatrix}\distc(t) \\ \diste(t) \end{bmatrix} \le \int_{\tau=0}^{t} \textrm{RHS}\Big(\tau,\begin{bmatrix}\distc \\ \diste \end{bmatrix}\Big) d\tau \doteq \begin{bmatrix}\bar\distc(t) \\ \bar\diste(t)\end{bmatrix}, \quad \distc(0) = \bar d_c(0),\ \diste(0) = \bar d_e(0).
	\end{equation}
\end{theorem}
Evaluating the integral in \eqref{eq:trk_bnd} is efficient as RHS is affine, so $\bar\distc$ and $\bar\diste$ can be readily used in planning (cf. Sec. \ref{sec:method_ofmp}). However, note that these tubes are only \textit{locally valid}, e.g., evaluating the tubes outside of $D_x$ will give incorrect values. We detail a set of validity conditions in Sec. \ref{sec:method_ofmp}, prove their sufficiency in Thm. \ref{thm:lmtd}, use them in our planner, and show in Sec. \ref{sec:results} that a baseline that ignores these conditions is unsafe. Finally, we close with a remark on how we estimate the constants in the bounds.

\begin{remark}[Estimating constants from data]\label{rem:lipschitz_estimation}
	The derived bounds depend on several constants that are unknown \textit{a priori}, such as $L_{\Delta k}$ and $L_{\hinv}$, and if $\bar\err_1$, $\bar\err_2$, or $\bar\err_3$ is being used, $\bar\err_1$, $L_p$, and $\{L_p, \disp\}$ also need to be estimated, respectively. As overapproximating each constant also yields valid (and looser) bounds, we use the i.i.d. validation set $\V$ to overestimate each constant via a sampling-based approach based on extreme value theory \cite{cdc}. This returns a value which overestimates the true constant with a user-desired probability $\delta$, where $\delta$ holds in the limit of infinite samples. See \cite{cdc, lipschitz_ral, weng2018evaluating} for details.
\end{remark}

\subsection{Optimizing CCMs and OCMs for output feedback}\label{sec:method_ccms_and_ocms}

We briefly discuss how we obtain the CCM/OCM that define the controller/observer; for space, we detail our method in App. \ref{sec:app_method_ccms_and_ocms}. We write two SoS programs to independently synthesize the CCM/OCM, which are approximately optimized to minimize their tube sizes. We search over polynomial CCMs and constant OCMs. For polynomial dual CCMs $\Wc(x)$, we also find a constant metric $\bar\Wc \succeq \Wc(x)$, for all $x$, in order to simplify constraint checking in Sec. \ref{sec:method_ofmp}. For linear systems, these SoS programs simplify to a standard semidefinite program (SDP), which scale to higher-dimensional systems.

\subsection{Solving the OFMP}\label{sec:method_ofmp}

\begin{algorithm*}\label{alg:rrt}
\scriptsize\KwIn{$x_I$, $\mathcal{G}$, $\theta$, $\S$, training error $\{e_i\}_{i=1}^{\Ndata}$, estimated constants, $\bar\distc(0)$, $\bar\diste(0)$, $\bar c$, $\bar e$}
\DontPrintSemicolon
\SetKwFunction{SampleState}{SampleState}
\SetKwFunction{SampleNode}{SampleNode}
\SetKwFunction{SampleCandidateControl}{SampleProposedControl}
\SetKwFunction{NearestNeighbor}{NN}
\SetKwFunction{IntegrateLearnedDyn}{IntegrateDyn}
\SetKwFunction{DCheck}{DCheck}
\SetKwFunction{ErrBnd}{ErrBnd}
\SetKwFunction{SteerInDEpsilon}{SteerInDEpsilon}
\SetKwFunction{OneStepReachable}{OneStep}
\SetKwFunction{Model}{Model}
\SetKwFunction{ConstructPath}{ConstructPath}
\SetKwFunction{InCollision}{InCollision}

$\T \leftarrow \{(x_I, \bar\distc(0), \bar\diste(0), 0)\}$ \tcp*{node: state, CCM/OCM Riem. dist. bound, time}
$\mathcal{P} \leftarrow \{(\emptyset, \emptyset)\}$\tcp*{parent: previous control/dwell time}

\While{\upshape True}{
	$(x_{\textrm{n}}, \bar d_c^n, \bar d_e^n, t_\textrm{n}) \leftarrow $ \SampleNode{$\T$} \tcp*{sample node from tree}
	$(u_\textrm{p}, t_\textrm{p}) \leftarrow $ \SampleCandidateControl() \tcp*{sample ctrl/dwell time}
	$(x_\textrm{p}^*(t), u_\textrm{p}^*(t)), t\in[t_n,t_n+t_p) \leftarrow$ \IntegrateLearnedDyn($x_{\textrm{n}}$, $u_\textrm{p}$, $t_\textrm{p}$)\tcp*{get extension}
	$(\bar d_c^n(t), \bar d_e^n(t)), t\in[t_n,t_n+t_p) \leftarrow$ \ErrBnd($\bar d_c^n$, $\bar d_e^n$, $x_\textrm{p}^*(t)$, $u_\textrm{p}^*(t)$, $\S$, $\{e_i\}$, $\theta$)\tcp*{new tube}
	$(b_L^c, b_L^e) \leftarrow (\bar d_c^n(t) \le \bar c, \bar d_e^n(t) \le \bar e), \forall t \in [t_n, t_n+t_p)$\tcp*{check upper bound}
	$b_c \leftarrow \Omega_c^n(t) \subseteq (D_c \cap D_r \cap \Xsafe), \forall t \in [t_\textrm{n}, t_\textrm{n}+t_\textrm{p})$\tcp*{check tracking tube}
	$b_e \leftarrow \Omega_c^n(t) \oplus (\Omega_e^n(t) \ominus \{x(t)\}) \subseteq (D_e \cap D_c), \forall t \in [t_\textrm{n}, t_\textrm{n}+t_\textrm{p})$\tcp*{chk. estimator tube}
	\lIf{$b_L^c \wedge b_L^e \wedge b_c \wedge b_e$}{$\T \leftarrow \T \cup \{(x_c^*(t_\textrm{n}+t_\textrm{p}), \bar d_c^n(t_\textrm{n}+t_\textrm{p}), \bar d_e^n(t_\textrm{n}+t_\textrm{p}), t_\textrm{p})\}$; $\mathcal{P} \leftarrow \mathcal{P} \cup \{(u_\textrm{p}, t_\textrm{n}+t_\textrm{p})\}$}
	\lElse{continue\tcp*[f]{add extension if all checks pass}}
	\lIf{$\exists t, \Omega_c^n(t) \subseteq \mathcal{G}$}{break; return plan \tcp*[f]{return if in $\mathcal{G}$}}
}
\caption{\textbf{C}ontraction-based \textbf{O}utput feedback \textbf{RRT} (\lmtds)}
\end{algorithm*}\vspace{-3pt}

Given the CCM, OCM, and the ability to compute tracking tubes, we can now solve the OFMP. Our solution builds upon a kinodynamic RRT \cite{lavalle2006planning}, though we note that the tubes derived in Sec. \ref{sec:method_bounds} are planner-agnostic. We grow a search tree $\T$ by integrating sampled controls held for sampled dwell-times until $\mathcal{G}$ is reached. To ensure we stay in $\Xsafe$ at runtime, we impose extra constraints on each candidate transition, which are informed by the tubes; this translates to a restriction on where $\T$ can grow (cf. Fig. \ref{fig:lmtd}).

\setlength{\intextsep}{5pt}%
\setlength{\columnsep}{5pt}%
\begin{wrapfigure}{R}{0.6\textwidth}
    \centering
    \includegraphics[width=0.95\textwidth]{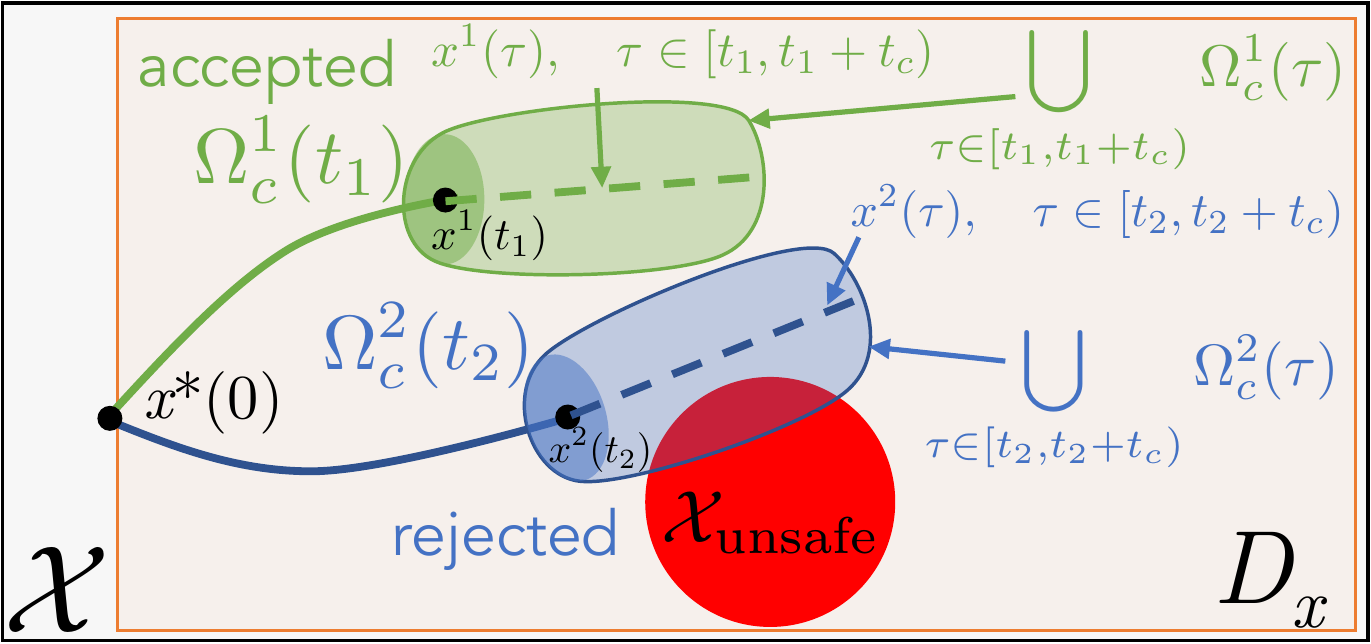}
    \caption{Visualization of Alg. \ref{alg:rrt}.}
    \label{fig:lmtd}
\end{wrapfigure}

To use the Riemannian distance bounds $\bar\distc(t)$ and $\bar\diste(t)$ from \eqref{eq:trk_bnd} in planning, recall that these bounds define sets centered around $x^*(t)$ and $x(t)$, $\Omega_c(t)$ and $\Omega_e(t)$, which $x$ and $\xhat$ are guaranteed to remain within. We can use these sets for collision and constraint checking. If the metric defining $\Omega(t)$ is constant, each $\Omega(t)$ defines an ellipsoid, i.e., $\Omega_c(t) = \{x(t) \mid (x(t) - x^*(t))^\top \Mc (x(t) - x^*(t)) \le \bar\distc(t)^2\}$ and $\Omega_e(t) = \{\xhat(t) \mid (\xhat(t) - x(t))^\top \Wo (\xhat(t) - x(t)) \le \bar\diste(t)^2\}$. If the metric is state-dependent (as is the case for some CCMs we use), we can use $\bar \Wc$ (see Sec. \ref{sec:method_ccms_and_ocms}) to obtain an ellipsoidal outer approximation of $\Omega_c(t)$: $\Omega_c(t) \subseteq \{x(t) \mid (x(t) - x^*(t))^\top (\bar\Wc)^{-1} (x(t) - x^*(t)) \le \bar\distc(t)^2 \} \doteq \tilde\Omega_c(t)$ that can ease constraint checking. Thus, we can guarantee \textit{at planning time} that \textit{in execution}, $x(t) \in \tilde\Omega_c(t)$, and $\xhat(t) \in \tilde\Omega_c(t) \oplus (\Omega_e(t) \ominus \{x(t)\})$, where $A \ominus B \doteq \{x - y \mid x \in A, y \in B\}$. As \eqref{eq:trk_bnd} defines $\Omega$ for \textit{any} nominal trajectory, we can quickly compute tubes along all edges in $\T$.  For instance, suppose we wish to extend from some node in $\T$, $x_n^*(t_n)$, which satisfies $d_c^n(t_n) \le \bar d_c^n(t_n)$ and $\diste^n(t_n) \le \bar d_e^n(t_n)$, to a candidate state $x_n^*(t_n+t_p)$ by applying control $u$ over $[t_n, t_n+t_p)$. Then, using \eqref{eq:trk_bnd}, we can obtain $\bar d_c^n(t)$ and $\bar d_e^n(t)$, for all $t\in[t_n, t_n+t_p)$, and to remain collision-free in execution, we require the induced $\tilde \Omega_c(t) \subseteq \Xsafe$; we check this in line 9 of our planner, Alg. \ref{alg:rrt}. Here, we assume obstacles are inflated to account for robot geometry.

To remain collision-free at runtime, we must add extra constraints on $\T$ to ensure the tubes are valid, as discussed in Sec. \ref{sec:method_bounds}. We describe these constraints now, and prove they are sufficient in Thm. \ref{thm:lmtd}. At a high level, the estimated constants, CCM, and OCM must be valid for any $x$ and $\xhat$ that can be reached at runtime. Thus, in line 8, we ensure $\distc(t)$ and $\diste(t)$ remain less than $\bar c$ and $\bar e$ for all time, so that $L_{\Delta k}$ \eqref{eq:lip} is valid. In line 9, we ensure that $\Omega_c(t) \subseteq D_c \cap D_r$, i.e., the system remains where the controller can contract $x$ towards $x^*$, and $\bar\err_i$ is valid. In line 10, we ensure $\xhat$ remains in $D_e \cap D_c$; this ensures that \eqref{eq:observer_reduced} contracts towards the true state $x$ via \eqref{eq:con_ccm}, and that a feasible feedback control exists in \eqref{eq:ufeas}; ensuring this at planning time (when we only know $x^*(t)$) requires a Minkowski sum of $\Omega_c$ and $\Omega_e \ominus \{x(t)\}$. Constraint-satisfying candidate extensions are added to $\T$ (line 11); else, they are rejected (line 12). This continues until the goal is reached (line 13). We visualize our planner (Fig. \ref{fig:lmtd}), \textbf{C}ontraction-based \textbf{O}utput feedback \textbf{RRT} (CORRT), detailed in Alg. \ref{alg:rrt}. Finally, Thm. \ref{thm:lmtd} shows our method ensures safety and goal reachability if all estimated constants are valid; as our estimates are probabilistically-valid, the overall guarantees are probabilistic (cf. Rem. \ref{rem:overall_correctness}):\vspace{-5pt}

\begin{theorem}[\lmtd correctness]\label{thm:lmtd}
	Assume that $L_{\Delta k}$, $L_{\hinv}$, and the estimated constants in $\bar\err_{\{1,2,3\}}$ are valid over their computed domains. Then Alg. \ref{alg:rrt} returns a trajectory $(x^*(t), u^*(t))$, which when tracked on the true system \eqref{eq:system} using $u(\xhat, x^*, u^*)$ with state estimates $\xhat$ generated by \eqref{eq:observer_reduced}, reaches $\mathcal{G}$ while satisfying $x(t) \in \Xsafe$, for all $t \in [0, T]$.
\end{theorem}

\section{Results}\label{sec:results}

We evaluate \lmtd on a 4D car with RGB-D observations, a 6D quadrotor with RGB observations, and a 14D acceleration-controlled 7DOF arm with RGB observations. All observations are rendered in PyBullet. We compare with three baselines; two are shared across experiments, so we overview them here. To show the need to plan where the CCM/OCM are valid and the error bounds are accurate, Baseline 1 (B1) plans using the tracking tubes from \eqref{eq:trk_bnd} inside Alg. \ref{alg:rrt} but is \textit{not constrained} to stay within $D$, i.e., the checks in line 8-10 of Alg. \ref{alg:rrt} are relaxed. To show the need to consider estimation error in planning, Baseline 2 (B2) assumes perfect state knowledge in computing its tubes, i.e., $\diste(t) \equiv 0$. All baselines execute with the same CCM/OCM as our method. See Table \ref{table:stats} for error statistics and the video \url{http://tinyurl.com/wafr22corrt} for visualizations.

\begin{table*}\centering
\begin{adjustbox}{max width=\textwidth}
\begin{tabular}{ c | c | c | c |c | c | c| c|c}
  & \lmtd trk. err. & \lmtd est. err. & B1 trk. err. & B1 est. err. & B2 trk. err. & B2 est. err.& B3 trk. err.& B3 est. err.\\\hline
  
  Car & 0.175 $\pm$ 0.117  & 0.032 $\pm$ 0.022 & 17.49 $\pm$ 79.86 & 143.4 $\pm$ 1202 & 1.520 $\pm$ 6.306 & 3.597 $\pm$ 19.90 & --- & ---\\
  \cellcolor{lightgray!50!}Quad & \cellcolor{lightgray!50!}0.151 $\pm$ 0.187 & \cellcolor{lightgray!50!}0.029 $\pm$ 0.028 & \cellcolor{lightgray!50!}39.30 $\pm$ 142.1 & \cellcolor{lightgray!50!}52.64 $\pm$ 185.9 & \cellcolor{lightgray!50!}40.56 $\pm$ 302.1 & \cellcolor{lightgray!50!}63.53 $\pm$ 424.1 & \cellcolor{lightgray!50!}--- &\cellcolor{lightgray!50!} ---\\
  Arm & 2.0e-4 $\pm$ 1.3e-5 & 0.053 $\pm$ 0.039 & 2.0e-4 $\pm$ 1.4e-5 & 0.145 $\pm$ 0.239 & --- & ---  &0.000 $\pm$ 0.000 &0.316 $\pm$ 0.249 \\
\end{tabular}
\end{adjustbox}\vspace{5pt}
\caption{Statistics on the tracking/estimation error reduction across all experimental results. ``Trk. err.'' $ = \Vert x^*(T) - x(T) \Vert/\Vert x^*(0) - x(0) \Vert $. ``Est. err.'' $ = \Vert \xhat(T) - x(T) \Vert/\Vert \xhat(0) - x(0) \Vert $. In each cell: average error $\pm$ standard deviation over all trials.}\vspace{-12pt}
\label{table:stats}
\end{table*}

\begin{figure}
\includegraphics[width=\textwidth]{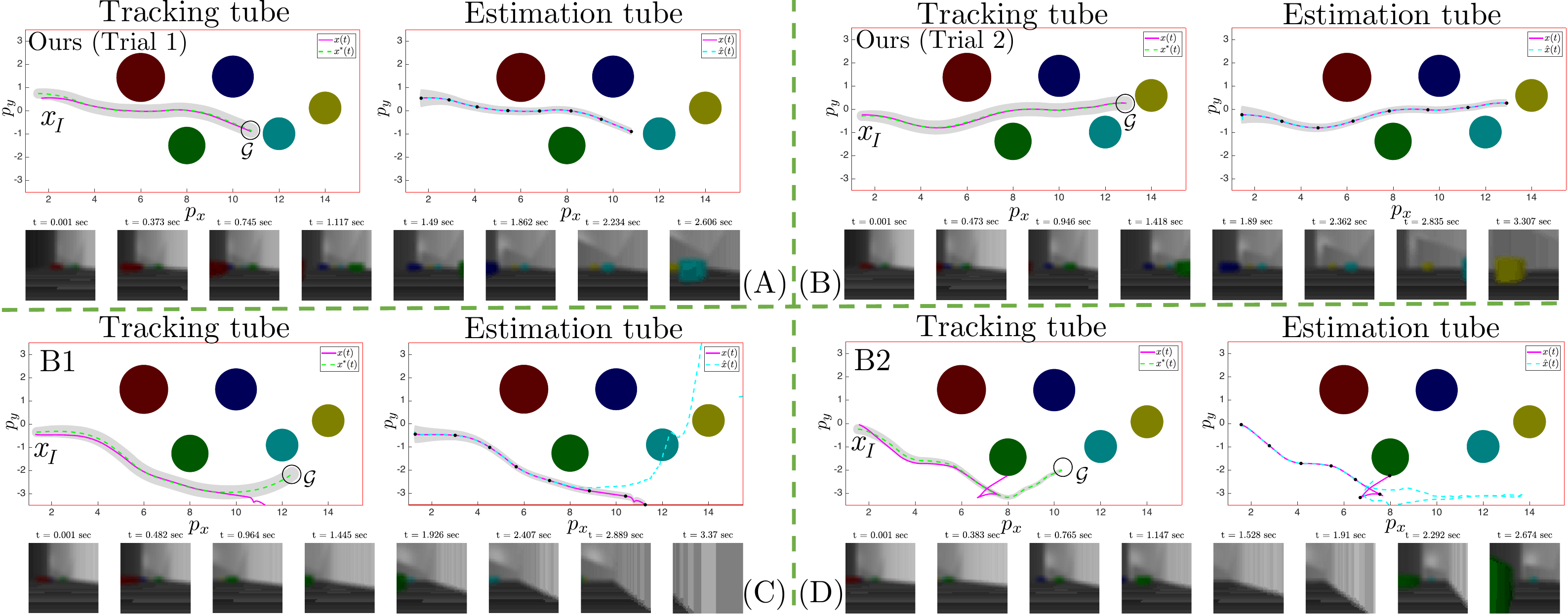}
\caption{4D car. Planned, executed, and estimated trajectories, overlaid with corresponding tracking and estimation tubes $\Omega_c(t)$ and $\Omega_e(t)$. For eight timesteps corresponding to the black dots on the $\Omega_e$ plot, we also show RGB component of the observations seen at runtime (bottom). A) and B): two examples of \lmtds, which safely reach the goal. C) and D): B1 and B2: both crash.} \label{fig:car}
\end{figure}

\subsubsection{4D nonholonomic car}
We consider a ground vehicle in an obstacle field (Fig. \ref{fig:environments}.A), governed by \eqref{eq:car}. The observations are given by 48x48 RGB-D images taken from a front-facing onboard camera (Fig. \ref{fig:environments}.A, inset); this makes $y \in \mathbb{R}^{9216}$. Three states can be directly inferred from a single image: $p_x$, $p_y$, and $\phi$. For this example, $\theta \in \mathbb{R}^5$ parameterizes the $p_y$-translation of each of the five obstacles. We are given $\Ndata = 250000$ datapoints to train the perception system $\hinv$, sampled uniformly from $\Cred D_p = [0, 13.5]\times[-2.5, -2.5]\times[-\pi/3,\pi/3]$ and $D_\theta = [0.5,1.5]\times[-1.5,-0.5]\times[0.5,1.5]\times[-1,0]\times[0,1]$. We model $\hinv$ as a fully-connected neural network, with five hidden layers of width 1024 and softplus activations. We use the method of Sec. \ref{sec:method_ccms_and_ocms} to obtain a constant CCM $\Mc$ with $\maxeig{\Mc} = 1$, $\mineig{\Mc} = 0.07$, and $\lambda_c = 2.5$, and a constant OCM $\Mo$ with $\maxeig{\Wo} = 5.44$, $\mineig{\Wo} = 0.05$, and $\lambda_o = 0.6$, where $D_c = (-\infty, \infty)^2 \times [-\pi/3, \pi/3] \times [2, 5] = D_e$. To compute our tubes in \lmtds, we use $\bar\err_3(x^*, \theta)$, since for this example $\Omega_c$ may be large. The constants are estimated to be $L_{\Delta k} = 3.28$, $L_{\hinv} = 0.05$, $L_p = 0.024$, and $\disp=0.69$. In computing our tubes, we assume $\Vert \dx \Vert \le 0.05$, $\bar\distc(0) = 0.2$, $\bar\diste(0) = 0.1$, and $\dy \in \mathbb{R}^{n_y}$ satisfies $\Vert \dy \Vert \le 0.25$. To simulate noisy depth images, $B_y$ is set to be a diagonal $n_y \times n_y$ matrix, with $0$ diagonal entries for RGB indices and $1$ for the depth indices.

We plan for 150 start/goals in $D$; our unoptimized implementation takes 2.5 minutes on average. This is done offline; the tracking controller is computed at real-time rates following Sec. \ref{sec:bound_trk_err} and \cite{sumeet_icra}. For each trial, the obstacle map $\theta$ is selected uniformly within $D_\theta$. See Table \ref{table:stats} for error statistics. Over all trials, our method ensures $x(t)$ and $x^*(t)$ always remain within the \lmtds-computed $\Omega_c(t)$ and $\Omega_e(t)$, respectively, and reduces the initial tracking/estimation error by a factor of $>5$ and $30$, respectively. In contrast, B1 violates its $\Omega_c(t)$ and $\Omega_e(t)$ in 90/150 and 101/150 trials, respectively, fails to reduce tracking/estimation error, and can crash. For instance, in Fig. \ref{fig:car}.C, the plan leaves $D_r$, causing observation error to increase (here, $\hat h$ is inaccurate, since it is not trained outside of $D_r$), destabilizing $\xhat$ (Fig. \ref{fig:car}.C, right), in turn destabilizing $x$, leading to the crash. Similarly, B2 violates its computed $\Omega_c$ in 60/150 trials (no $\Omega_e(t)$ is computed for B2, as it assumes perfect state information), fails to shrink tracking/estimation errors, leading to crashes (see Fig. \ref{fig:car}). As in B1, this crash also arises from observation error. Overall, this experiment suggests that \lmtd ensures safe goal-reaching for nonholonomic systems using RGB-D data, and that it generalizes to different environments (i.e., obstacle layouts), while baselines are unsafe.

\begin{figure}[t]
\includegraphics[width=\textwidth]{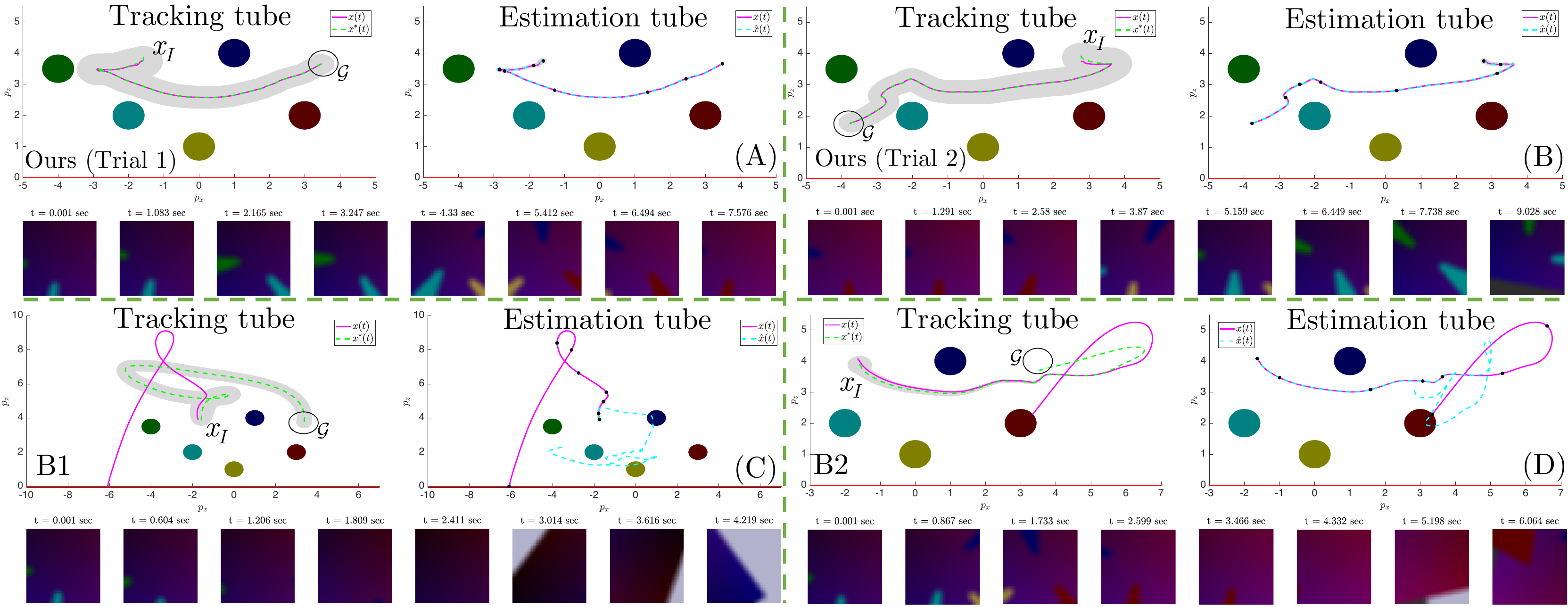}
\caption{6D quadrotor. Planned, executed, and estimated trajectories, overlaid with $\Omega_c(t)$ and $\Omega_e(t)$. Snapshots of the runtime observations are shown (bottom). A) and B): two examples of \lmtds, which safely reach the goal. C) and D): B1 and B2: both crash.} \label{fig:pvtol}
\end{figure}

\subsubsection{6D quadrotor}

We consider a planar quadrotor in an obstacle field (Fig. \ref{fig:environments}.B), governed by \eqref{eq:pvtol}. The observations are given by 48x48 RGB images taken from a front-facing onboard camera (Fig. \ref{fig:environments}.B, inset); this makes $y \in \mathbb{R}^{6912}$. Three states can be directly inferred from an image: $p_x$, $p_z$, and $\phi$. Here, we consider a single set of map configurations, i.e., $\theta$ is a singleton. We are given $\Ndata = 140000$ datapoints to train $\hinv$, sampled uniformly from $\Cred D_p = [-4.5, 4.5]\times[0.5, 4.5]\times[-\pi/4,\pi/4]$. We model $\hinv$ as a fully-connected neural network, with five hidden layers of width 1024 and ReLU activations. Using the method of Sec. \ref{sec:method_ccms_and_ocms}, we obtain a polynomial CCM $\Mc$ with $\maxeigc{\Mc} = 6.55$, $\mineigc{\Mc} = 0.22$, and $\lambda_c = 0.8$, and a constant OCM $\Mo$ with $\maxeig{\Wo} = 8.13$, $\mineig{\Wo} = 0.1$, and $\lambda_e = 0.7$, where $D_c = (-\infty, \infty)^2 \times [-\pi/3, \pi/3] \times [-4.5, 4.5] \times [-1, 1] \times [-2, 2]$ and $D_e = (-\infty, \infty)^2 \times [-\pi/4, \pi/4] \times [-5, 5] \times [-2.5, 2.5] \times [-2.5, 2.5]$. To update our tracking tubes in \lmtds, we found it sufficient to use the first error bound $\bar\err_1$, which we estimate to be $0.008$, and $L_{\Delta k} = 3.6$. In computing our tubes, we assume $\Vert \dx \Vert \le 0.0125$, $\bar\distc(0) = 0.15$, $\bar\diste(0) = 0.1$, and noiseless images $\Vert \dy \Vert = 0$.

We plan for 150 start/goals in $D$, taking 6 minutes on average (see Table \ref{table:stats} for statistics). Across all trials, \lmtd ensures $x(t)$ and $\xhat(t)$ stay inside the \lmtds-computed tubes $\Omega_c(t)$ and $\Omega_e(t)$, respectively, and reduces the initial tracking/estimation error by a factor of $>6$ and $34$. In contrast, B1 violates its computed $\Omega_c(t)$ and $\Omega_e(t)$ in 61/150 and 76/150 trials, respectively, fails to reduce error, and can be unsafe (see Fig. \ref{fig:pvtol}). Similarly, B2 violates its $\Omega_c$ in 142/150 trials. We show concrete examples of this in Fig. \ref{fig:pvtol}.C-.D; the plans in both cases exit $D_r$, moving to $p_x$ and $p_z$ values outside of the $[-4.5,4.5]\times[0.5,4.5]$ training range, leading to high $\hinv$ error. The plans also take overly-aggressive turns that bring the velocities outside of $D_e$ and $D_c$; this further destabilizes the system, causing crashes in both cases. Overall, this experiment suggests the need to ensure that $\hinv$, the CCM, and the OCM are correct, and that \lmtd ensures this to guarantee safety for underactuated systems via RGB observations.

\begin{figure}
\includegraphics[width=0.95\textwidth]{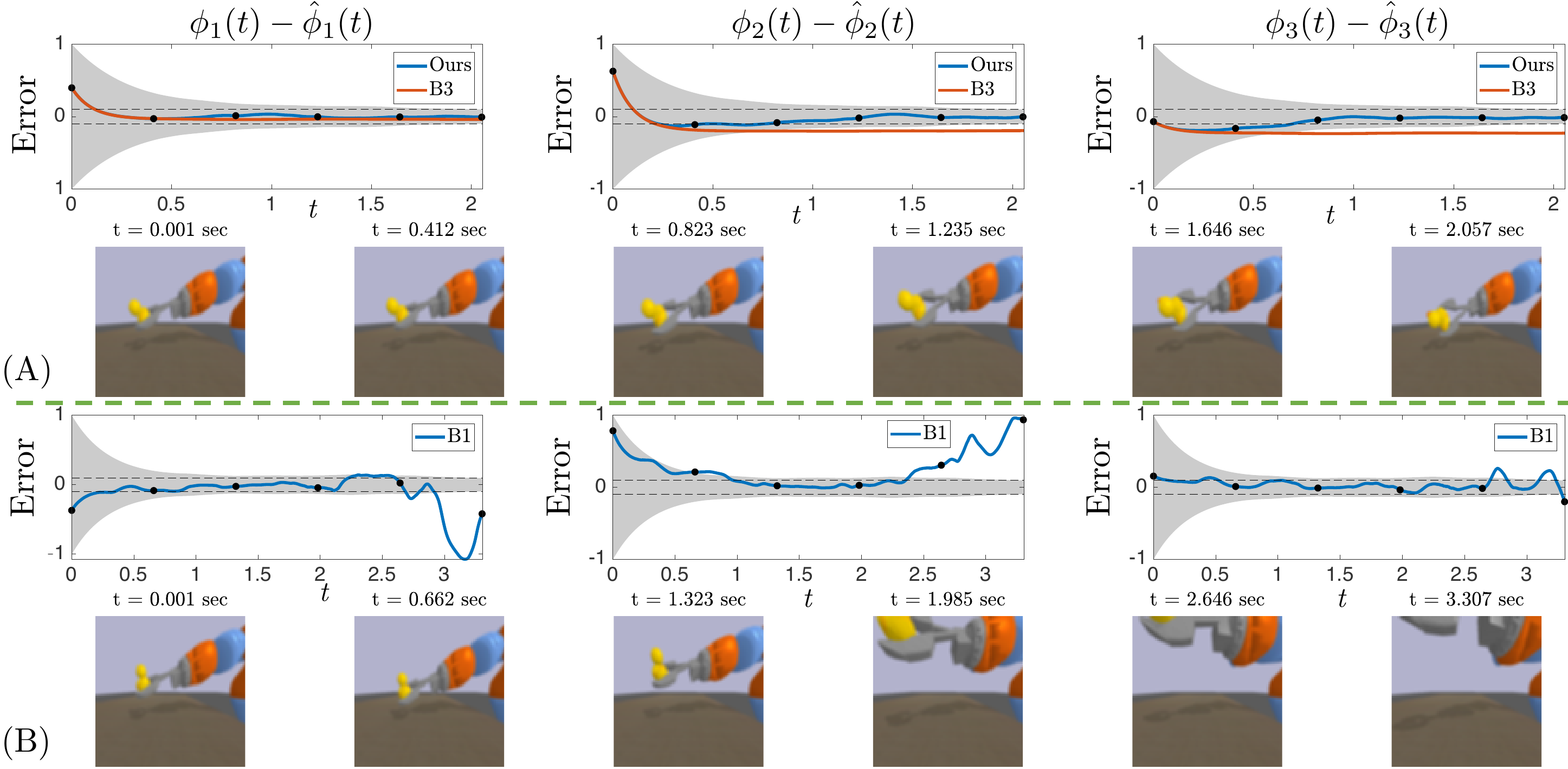}
\caption{7DOF arm. State estimate error, overlaid with $\Omega_e(t)$ (in gray). Runtime observations are shown (bottom). A): when using \lmtds, the state estimate error remains in $\Omega_e(t)$ and achieves $|\hat\phi_i(T) - \phi_i(T)| \le 0.1$. B3 fails to meet this requirement. B) B1 also fails the 0.1 requirement.} \label{fig:arm}
\end{figure}

\subsubsection{17D manipulation task} We consider an acceleration-controlled 7DOF Kuka arm, where each joint follows double integrator dynamics \eqref{eq:arm_14d}, which is grasping an object (a rubber duck) with an unknown orientation relative to the end effector. We assume slight noise in the dynamics \eqref{eq:arm_14d}, $\bardx = 0.0125$, due to the weight of the object. Our goal is to estimate the unknown orientation, represented as three Euler angles $\{\phi_i\}_{i=1}^3$, using our observer \eqref{eq:observer_reduced}, given 80x80 RGB images (Fig. \ref{fig:environments}.C) of the arm and grasped object (see Fig. \ref{fig:environments}.C, inset); this makes $y \in \mathbb{R}^{19200}$. We may also plan motions for the arm to improve the quality of the observations/state estimates, though in doing so, we also need to counteract the dynamics error. Our overall goal is to guarantee our final estimate of the relative orientation satisfies $| \phi_i(T) - \hat\phi_i(T) | \le 0.1$, $i = 1,2,3$. 

We assume that the joint angles and velocities can be perfectly estimated (i.e., directly measured), given the accuracy of the Kuka joint encoders, focusing instead on estimating the unknown $\{\phi_i\}_{i=1}^3$ and controlling $j$ and $\dot j$ (the joint angles and velocities) using our method. We assume the object is rigidly attached to the gripper, such that its relative orientation is constant over time. Combining $\{\phi_i\}_{i=1}^3$ and the 14D model, the full state of the system is 17D \eqref{eq:arm_17d}, i.e., $x = [\phi_1, \phi_2, \phi_3, j_1, \ldots, j_7, \dot j_1, \ldots \dot j_7]^\top$. To train $\hinv$, we note that $\{\phi_i\}_{i=1}^3$ can all be estimated from the image. For this example, since $j$ is known and affects the generated $y$, we design $\hinv$ to take as input $y \in \mathbb{R}^{19200}$ and $j \in \mathbb{R}^7$ (i.e., $j$ plays the role of $\theta$) and to output $\{\phi_i\}_{i=1}^3$. We are given $\Ndata = 62500$ datapoints to train $\hinv$, where $\{\phi_i\}_{i=1}^3$ are sampled uniformly from $[-\pi/3,\pi/3]^3$ and $j$ is sampled uniformly from $[-0.05, 0] \times [0, 0.05] \times [0.15, 0.32] \times [-1.83, -1.69] \times [-0.05, 0.05]^2 \times [-\pi/3, \pi/3]$. We model $\hinv$ as a fully-connected neural network, with five hidden layers of width 1024 and softplus activations. We compute a constant CCM for the 14D subsystem: CCM synthesis for the full 17D system fails, as the $\{\phi_i\}_{i=1}^3$ are not controllable due to the rigid attachment. Since the arm dynamics are linear, the CCM optimization simplifies to a standard semidefinite program that can be quickly solved. We compute a constant OCM for the full 17D system, to enable estimation of $\{\phi_i\}_{i=1}^3$. Using the method of Sec. \ref{sec:method_ccms_and_ocms}, we obtain a CCM $\Mc$ with $\maxeig{\Mc} = 100$, $\mineig{\Mc} = 2.81$, and $\lambda_c = 2.89$, and a constant OCM $\Mo$ with $\maxeig{\Wo} = \mineig{\Wo} = 0.1$, and $\lambda_e = 9.5$. As the dynamics are linear, a constant CCM/OCM holds globally, i.e., $D_e = D_c = \X$. To update the tubes in \lmtds, we use $\bar\err_2(x^*,j)$, where we estimate $L_p = 2.45$. Since $j$ and $\dot j$ are known, no error arises from incorrect state estimates; thus, $L_{\Delta k}$ does not need to be estimated. We assume $\bar\distc(0) = 10^{-3}$, $\bar\diste(0) = 0.32$, and noiseless images $\Vert \dy \Vert = 0$.

We plan 100 trajectories in $D$ from various initial $j$, $\dot j$, and orientation estimates, taking 45 seconds on average. We summarize the error statistics in Table \ref{table:stats}. Across all trials, when planning with \lmtds, $x$ and $\xhat$ always remain within the computed tubes $\Omega_c(t)$ and $\Omega_e(t)$; the CCM keeps the tracking error very small, and the OCM shrinks the error by a factor of $>18$. Crucially, if a plan is found where $\Omega_e(T)$ satisfies the estimation accuracy threshold, we can ensure our true state estimate satisfies $| \phi_i(T) - \hat\phi_i(T) | \le 0.1$, $i = 1,2,3$. We are able to find plans that achieve this threshold for 100/100 trials. We compare with two baselines in this example: B1 (as described before), and B3, which keeps the arm stationary and runs \eqref{eq:observer_reduced} for the same duration as the plan computed using \lmtds. The purpose of B3 is to show that the actions taken by the \lmtd plan help to reduce estimation error. In contrast to \lmtds, B1 violates its computed $\Omega_e(t)$ in 44/100 trials and can fail to achieve the required estimation accuracy, only satisfying the 0.1 threshold in 79/100 trials (see Fig. \ref{fig:arm}). One failure example is shown in Fig. \ref{fig:arm}.B: the arm moves too close to the camera (outside of $D_r$), causing the duck to fall out of frame. This causes a sharp increase in $\hinv$ error, since $\phi_i$ cannot be observed; this destabilizes \eqref{eq:observer_reduced}, leading to a failure to satisfy the 0.1 threshold. Note that B1 does not violate $\Omega_c$; this is because the controller is not a function of the incorrect $\phi_i$ estimates. Similarly, B3 often fails to satisfy the 0.1-estimation accuracy threshold, only satisfying it in 7/100 trials (see Fig. \ref{fig:arm}.A for a failure example). This shows that passively estimating $\phi_i$ without moving the arm cannot achieve the needed estimation accuracy; instead, the arm must be moved towards regions with smaller perception error. Overall, this experiment suggests the applicability of our approach on high-dimensional systems, that it can design actions that improve state estimates, and that our approach can plan paths that guarantee a desired level of state estimation accuracy.

\section{Discussion and Conclusion}\label{sec:conclusion}

We present a motion planning algorithm for control-affine systems that enables safe tracking at runtime using an output feedback controller with image observations as input. To achieve this, we learn a perception system and use it in an OCM and CCM-based output feedback control loop. We derive tracking tubes for the closed-loop system and use them within an RRT-based planner to compute plans that theoretically guarantee safe goal-reaching at runtime. Our results empirically validate this safety guarantee, and show that ignoring the effects of state estimation error and the local validity of the perception system/estimator/controller can lead to unsafe behavior.

Our method has some weaknesses which reveal directions for future work. While the large dataset $\S$ used to train $\hinv$ is easy to gather in simulation, sim-to-real is then needed for $\hinv$ to transfer to the real world. Thus, in future work, we will combine synthetic, domain-randomized perception data with a small real-world labeled dataset to train generalizable perception modules that have calibrated estimates of the sim-to-real error. Our method also assumes noiseless training data, to ensure $L_p$ is finite; in the future, we wish to relax this by investigating Lipschitz constant estimation methods robust to input noise \cite{calliess2014conservative} . Another drawback is the conservativeness of using worst-case disturbances; to mitigate this, we will integrate stochastic contraction \cite{DBLP:journals/corr/abs-2106-05635} into our method. Finally, we require $\theta$ to be known; in future work, we will aim to jointly estimate $\theta$ and $x$ with similar convergence guarantees.

\bibliographystyle{splncs04}
\bibliography{references}

\newpage
\appendix
\renewcommand{\theequation}{\thesection.\arabic{equation}}
\setcounter{equation}{0}

\begin{minipage}{\linewidth}
\centering{\huge{\textbf{Appendix}}}
\end{minipage}

\section{Trusted domain visualizations}\label{sec:app_viz}

\begin{figure}
\includegraphics[width=0.9\textwidth]{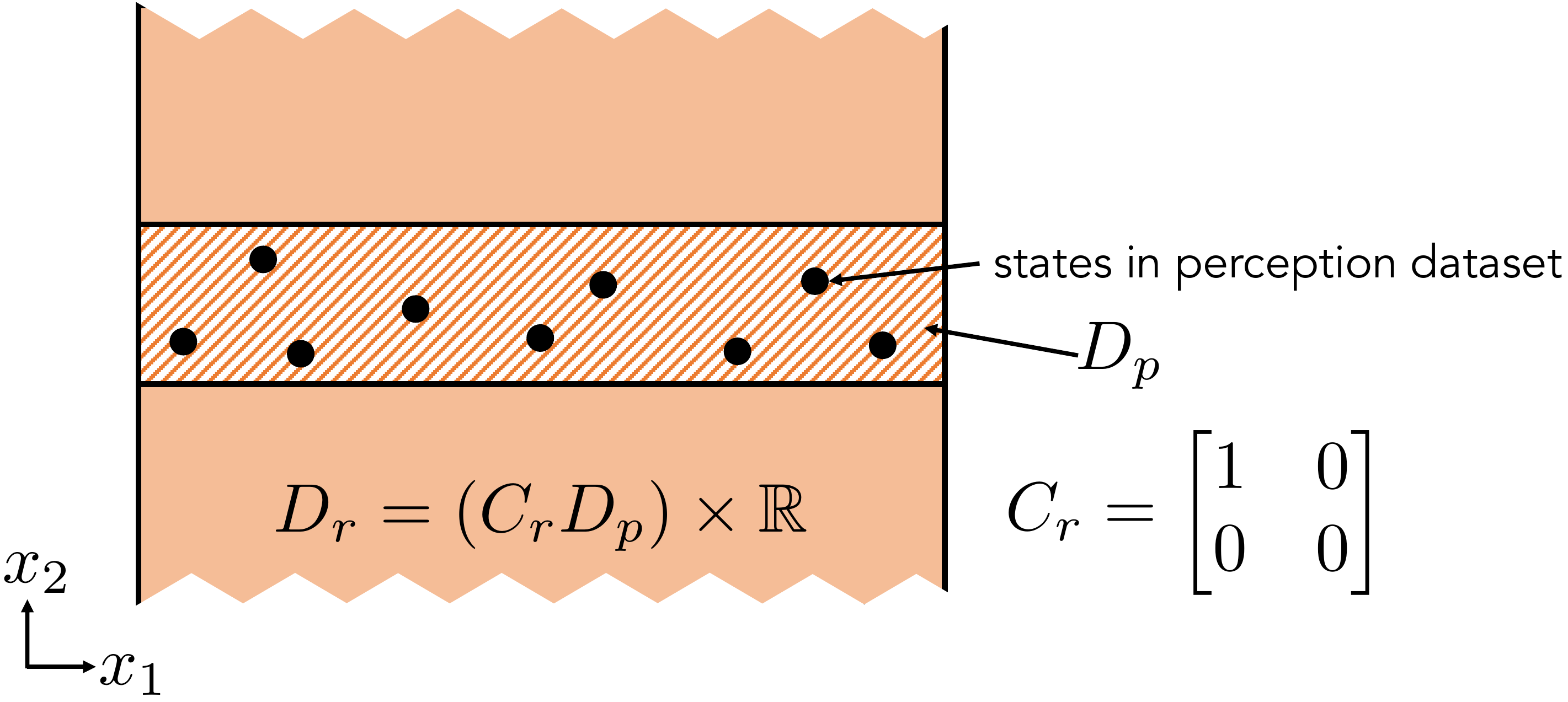}
\vspace{-8pt}
\caption{An example of how $D_r \subseteq \X$ is constructed.} \label{fig:dr}
\end{figure}

In Fig. \ref{fig:dr}, we visualize for a toy example how we construct $D_r$, which is a critical set upon which we define the trusted domain $D_x$.

\section{Bounding estimation error (expanded)}\label{sec:app_bounds}

How can we bound the learned perception module error $\err(x,\theta) \doteq \Vert \hinv\big(\h(x, \theta), \theta\big) - \Cred x \Vert$ over $D_r \times D_\theta$? We describe three options, each with their own strengths/drawbacks. The first and simplest option is to estimate a \textit{uniform bound} $\bar\err_1$ on the error:
\begin{equation}\label{eq:err_bnd_1}
	\err(x, \theta) \le \bar \err_1,\quad \forall (x,\theta) \in D_r \times D_\theta.
\end{equation}
This works well if the error is consistent across $D_r \times D_\theta$; however, it will be conservative if there are any error spikes.
A second option is to derive a \textit{spatially-varying bound} on $\err(x,\theta)$. To do so, we first estimate the Lipschitz constant of $\err$, $L_p$, over $D_r \times D_\theta$:
\begin{equation}
	\Vert \err(\tilde x, \tilde \theta) - \err(\check x, \check \theta) \Vert \le L_p \Vert (\tilde x, \tilde \theta) - (\check x, \check \theta) \Vert,\quad \forall (\tilde x, \tilde \theta), (\check x, \check \theta) \in D_r \times D_\theta.
\end{equation}
Denote $\S_{x\theta} = \{(x_i,\theta_i)\}_{i=1}^{\Ndata}$. We can then write this bound $\bar\err_2(x^*, \theta)$  (cf. Fig. \ref{fig:bounds}.B for visuals), which crucially is an explicit function of the plan $x^*$, and can thus directly guide our planner Alg. \ref{alg:rrt}:
\begin{theorem}[$\bar\err_2(x^*, \theta)$]: Recall that $\distc(t) = \distc(x^*(t), x(t))$ is the Riemannian distance between the nominal and true state at some time $t$. An upper bound on the perception error $\err(x,\theta)$ that scales linearly with $\distc(t)$ can be written as:
	\begin{equation}\label{eq:bnd2}\small
				\err(x,\theta) \le L_p \distc/\sqrt{\mineigc{\Mc}} + \textstyle\min_{1 \le i \le \Ndata}\{L_p (\Vert (x^*,\theta) - (x_i,\theta_i) \Vert) + \err_i \} \doteq \bar\err_2(x^*,\theta)
	\end{equation}
\end{theorem}
\begin{proof}
	For some training point $(x_i, \theta_i) \in (\S_{x\theta})$, we can calculate $\err_i \doteq \Vert \hinv(h(x_i,\theta_i),\theta_i) - \Cred x_i\Vert$ as its training error. Using $\err_i$ and $L_p$, we can bound the error at a query $(x, \theta) \in D_r \times D_\theta$, as $\err(x, \theta) \le L_p \Vert (x,\theta) - (x_i,\theta_i) \Vert + \err_i$. This holds for all data, so we can tighten the bound by taking the pointwise minimum of the bounds over all datapoints:
\begin{equation}\label{eq:lipschitz_err}
	\err(x,\theta) \le \textstyle\min_{1 \le i \le \Ndata}\{L_p \Vert (x,\theta) - (x_i,\theta_i) \Vert + \err_i \}
\end{equation}
As written, \eqref{eq:lipschitz_err} is only implicitly a function of the plan, via the relationship between $x$ and $x^*$, and cannot directly guide planning, since $x$ is unknown at planning time. We can make this bound an explicit function of the plan $x^*$ by rewriting it as follows:
\begin{equation}\small
	\begin{array}{ll}
		\err(x,\theta) &\le \textstyle\min_{1 \le i \le \Ndata}\{L_p \Vert (x,\theta) - (x_i,\theta_i) \Vert + \err_i\} \\
		&\le \textstyle\min_{1 \le i \le \Ndata}\{L_p (\Vert (x,\theta) - (x^*,\theta) \Vert + \Vert (x^*,\theta) - (x_i,\theta_i) \Vert) + \err_i \} \\
		&\le \frac{L_p \distc}{\sqrt{\mineigc{\Mc}}} + \textstyle\min_{1 \le i \le \Ndata}\{L_p (\Vert (x^*,\theta) - (x_i,\theta_i) \Vert) + \err_i \} \doteq \bar\err_2(x^*,\theta) \\
	\end{array}\hspace{-20pt}
\end{equation}
The second inequality follows from the triangle inequality, and the third inequality by applying $\sqrt{\mineigc{\Mc}}\Vert x_1 - x_2 \Vert \le \distc(x_1, x_2) \le \sqrt{\maxeigc{\Mc}}\Vert x_1 - x_2\Vert$.
\end{proof}

Instead of bounding the error over the whole domain (as in $\bar\err_1$), $\bar\err_2(x^*,\theta)$ is tighter as it only bounds the error over the tube, $\Omega_c(t)$. However, due to the leading term in \eqref{eq:bnd2}, $\bar\err_2(x^*,\theta)$ scales linearly with $\distc$, even if the true error $\err(x,\theta)$ is relatively constant, making it loose for large $\distc$. Thus, we derive a third error bound to mitigate this, $\bar\err_3(x^*,\theta)$ (cf. Fig. \ref{fig:bounds}.C). Overall, this third error bound $\bar\err_3(x^*, \theta)$ is useful for large $\distc$, as there is no direct scaling with $\distc$; however, it can be conservative if the data has high dispersion. We describe this in more detail in the following, after some definitions and proving a property of the dispersion. We provide a more detailed visualization of our bound in Fig. \ref{fig:bound_3}.

\vspace{5pt}
\begin{figure}
\includegraphics[width=0.9\textwidth]{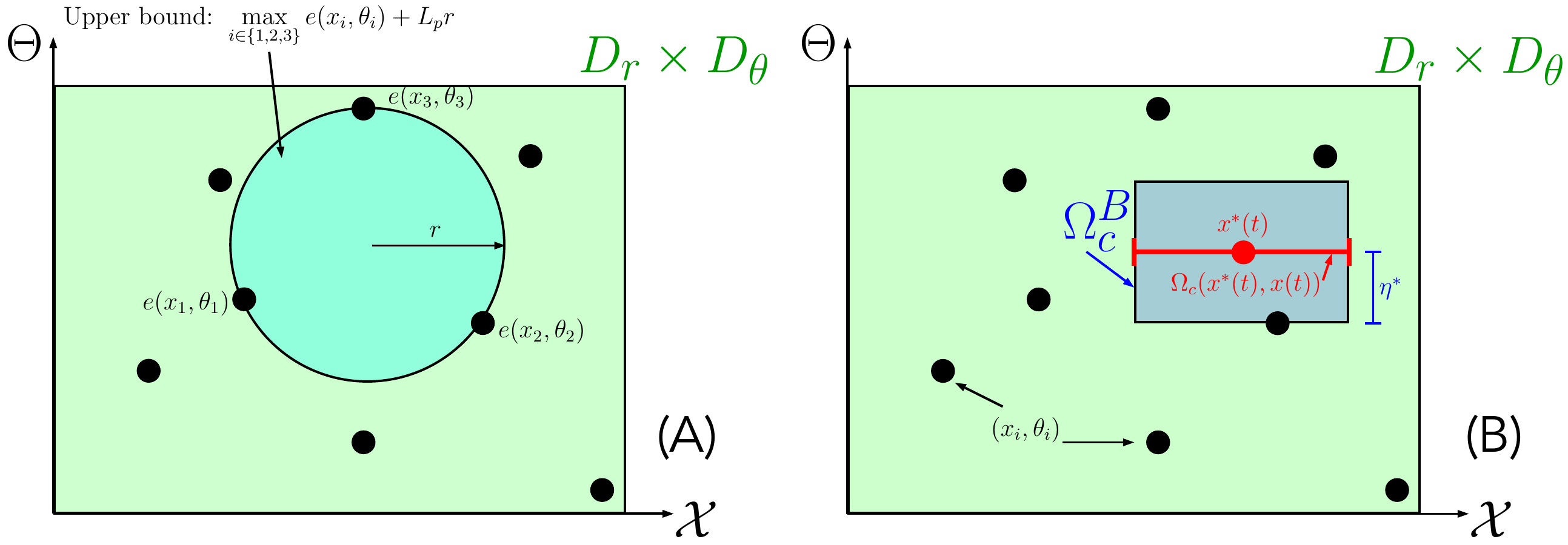}
\caption{(A) Visualization of the dispersion; together with the Lipschitz constant, it can bound the error within the blue set. (B) A visualization of our set construction in $\bar\err_3(x^*, \theta)$. } \label{fig:bound_3}
\end{figure}\vspace{5pt}

Define $\B_{r}(x) = \{\tilde x \mid \Vert\tilde x - x\Vert < r\}$. and the dispersion of a finite set $A$ contained in a set $B$, $\disp(A,B)$: 
\begin{equation}\label{eq:dispersion}
	\begin{array}{lcl}
		\disp(A, B) \doteq & \displaystyle\sup_{r \ge 0, x \in B} & r \\
		&\textrm{s.t.} & \B_r(x) \cap A = \emptyset, \quad \B_r(x) \subseteq B \\
	\end{array}
\end{equation}
i.e., the radius of the largest open ball inside $B$ that does not intersect with $A$. The following property of the dispersion will be useful to us in deriving $\bar\err_3(x^*,\theta)$.

\begin{lemma}[Dispersion of a subset]\label{lem:dispersion}
	Consider a finite set $X$, a compact set $Y$, where $X \subseteq Y$, and a compact subset $Z \subseteq Y$. Then, $\disp(X\cap Z, Y \cap Z) \le \disp(X, Y)$.
\end{lemma}
\begin{proof}
	Consider a ``sub-problem'' of \eqref{eq:dispersion}, which finds the largest open ball satisfying the constraints of \eqref{eq:dispersion} when the ball center is fixed to $x \in Z$:
	\begin{equation}\label{eq:dispersion_fixedx}
	\begin{array}{lcl}
		\disp_x(A, B) \doteq & \displaystyle\sup_{r \ge 0} & r \\
		&\textrm{s.t.} & \B_r(x) \cap B = \emptyset, \quad \B_r(x) \subseteq A \\
	\end{array}
\end{equation}
Compare the value of $\disp_x(X, Y)$ and $\disp_x(X \cap Z, Y \cap Z)$ for any fixed $x \in Z$. Note that the feasible set of \eqref{eq:dispersion_fixedx} when $A = X \cap Z$ and $B = Y \cap Z$ is contained within the feasible set of \eqref{eq:dispersion_fixedx} when $A = X$ and $B = Y$: any ball which is contained in $Y \cap Z$ and does not intersect $X \cap Z$ is also contained in $Y$, and does not intersect $X$. Thus, $\disp_x(X \cap Z, Y \cap Z) \le \disp_x(X, Z)$.

Now, consider the original problem \eqref{eq:dispersion}, which can be rewritten as $\disp(A, B) = \sup_{x \in B} \disp_x(A, B)$. We have that $\disp(X, Y) \ge \disp(X \cap Z, Y \cap Z)$, since $Y$ is a superset of $Y \cap Z$ and thus the feasible set when $B = Y$ is larger than when $B = Y \cap Z$.

\end{proof}

As an abuse of notation, we refer to $\disp$ without any arguments as shorthand for $\disp(\S_{x\theta}, D_r \times D_\theta)$. Let $\eta^*(t) = \min_{\eta \ge 0, \S_{x\theta} \cap (\Omega_c(t) \times \B_{\eta}(\theta)) \ne \emptyset} \eta$ be the radius of the smallest ball around $\theta$, $\B_{\eta^*(t)}(\theta)$, such that $(\Omega_c \times \B_{\eta^*(t)}(\theta)) \cap (D_r \times D_\theta) \doteq \Omega_c^B$ contains some datapoint. Then, we write our third error bound $\bar\err_3(x^*,\theta)$:
\begin{theorem}[$\bar\err_3(x^*, \theta)$]An upper bound on the perception error $\err(x,\theta)$ can be written based on buffering the local training errors:
	\begin{equation}
		\err(x, \theta) \le L_p \disp + \textstyle\max_{(x_i,\theta_i) \in \S_{x\theta} \cap \Omega_c^B} \err(x_i, \theta_i) \doteq \bar\err_3(x^*, \theta)
	\end{equation}
\end{theorem}

\begin{proof}

Finally, with an abuse of notation, denote $L_p(A)$ as the local Lipschitz constant of $\err(x,\theta)$ in $x$ and $\theta$, valid for all $(x,\theta) \in A$. 

	\begin{equation}\small\label{eq:err_bnd_3}
	\begin{array}{ll}
		\err(x, \theta) &\le \sup_{\tilde x \in \Omega_c} \err(\tilde x, \theta) \le \sup_{(\tilde x, \tilde \theta) \in \Omega_c^B} \err(\tilde x, \tilde \theta) \\
		&\le L_p(\Omega_c^B) \disp(\S_{x\theta} \cap \Omega_c^B, \Omega_c^B) + \max_{(x_i,\theta_i) \in \S_{x\theta} \cap \Omega_c^B} \err(x_i, \theta_i) \\
		&\le L_p(D_r \times D_\theta) \disp(\S_{x\theta}, D_r \times D_\theta) + \max_{(x_i,\theta_i) \in \S_{x\theta} \cap \Omega_c^B} \err(x_i, \theta_i) \doteq \bar\err_3(x^*, \theta),\\
	\end{array}\hspace{-20pt}
\end{equation}
where the second inequality follows from the definition of dispersion and the fourth follows from a property of dispersion shown in Lem. \ref{lem:dispersion}.
\end{proof}

Intuitively, $\bar\err_3(x,\theta)$ uses the training errors $e_i$ for points inside $\Omega_c(t)$ and buffers them with the Lipschitz constant of the error and the data density to upper bound the error inside $\Omega_c(t)$ (see Fig. \ref{fig:bounds}). For this to make sense, there must be at least one datapoint in the considered set; however, in general, the external parameter $\theta \in D_\theta$ given as input to the OFMP may not be precisely in the dataset (as the dataset is just a finite sampling of $D_\theta$). There will, however, be datapoints with $\theta_i$ close by $\theta$. Thus, for the maxima and suprema terms to be well-defined, we must buffer $\Omega_c(t)$ in the $\theta$ coordinates until at least one datapoint lies in $\Omega_c(t) \times \B_{\eta}(\theta)$, for some buffer radius $\eta \ge 0$. In the proof, we make the relaxation in the last inequality so that we only have to estimate one Lipschitz constant and one dispersion, instead of needing to calculate them for each $\Omega_c(t)$; this simplifies the estimation of the constants discussed in Rem. \ref{rem:lipschitz_estimation}.

Each of these three bounds on $\err(x,\theta)$ can be plugged into \eqref{eq:dq_derivation} to upper bound the integral in \eqref{eq:re}. As we use constant $\Mo$ and $\rho$ in the results, this simplifies the integral to $\Vert R_e \dq(t) \Vert$; we prove a bound on $\Vert R_e \dq(t) \Vert$ in Lemma \ref{lem:de}.

\section{Proofs}\label{sec:app_proofs}

In this appendix, we present the proofs of the theoretical results in the main body of the paper; for convenience, we have copied the theorem statements here.

\begin{lemma}[$\dot d_c(t)$]
The integral term in \eqref{eq:rc} can be bounded as
	\begin{equation}\label{eq:dc_bound_app}
		\textstyle\int_0^1 \Vert R_c(\gamma_c^t(s)) \dc(t) \Vert ds \le \sqrt{\maxeigc{\Mc}}\bardx + L_{\Delta k} \diste.
	\end{equation}
\end{lemma}

\begin{proof}
Recall the form of $\dc$ from \eqref{eq:dc}. First consider the third term in \eqref{eq:dc} (the dynamics error). By using $\Vert\dx\Vert \le \bardx$ and $\Vert R_c(\cdot) \Vert \le \sqrt{\maxeigc{\Mc}}$, we can pull $\dx$ and $R_c(\cdot)$ out of the integral to form the first term in the lemma statement. The second term follows by performing a max over geodesic parameters $s \in [0,1]$ and substituting the appropriate arguments into \eqref{eq:lip}. Specifically, we have the following chain of relations:
\begin{equation*}
	\begin{array}{ll}
	\int_0^1 \Vert R_c(\gamma_c^t(s)) B(u(\xhat, x^*, u^*) - u_\textrm{closest})\Vert ds \\ 
	\quad \quad \le \int_0^1 \max_{s\in[0,1]} \Vert R_c(\gamma_c^t(s)) B(u(\xhat, x^*, u^*) - u_\textrm{closest})\Vert ds \\
	\quad \quad = \int_0^1 \Delta k(\xhat, x, x^*, u^*) ds = \Delta k(\xhat, x, x^*, u^*)\\
	\quad \quad = |\Delta k(\xhat, x, x^*, u^*) - \Delta k(x, x, x^*, u^*)| \le L_{\Delta k} \diste(\xhat, x) = L_{\Delta k} \diste
\end{array}
\end{equation*}
The first inequality holds via upper-bounding the norm over the geodesic parameters, the first equality holds by substituting the definition of $\Delta k$, the second equality holds as $\Delta k$ is not a function of $s$, the third equality holds since $\Delta_k(x,x,x^*,u^*) = 0$, and the final inequality holds from the definition of our Lipschitz constant.

\end{proof}

\begin{lemma}[$\dot d_e(t)$]
Let $\bar\sigma(B_y)$ denote the maximum singular value of $B_y$. For constant $\rho$ and $\Mo$, the integral in \eqref{eq:re} simplifies to $\Vert R_e \dq(t) \Vert$ and can be bounded as:
	\begin{equation}\label{eq:lem2_app}\small
		\Vert R_e \dq(t) \Vert \le \sqrt{\maxeig{\Wo}}\bardx + \textstyle\frac{1}{2}\rho\maxeig{\Mo}^{1/2}\big(L_{\hinv}\sqrt{\maxsing{B_y}}\bardy + \bar\err_{\{1,2,3\}}(x^*,\theta)\big)
	\end{equation}
\end{lemma}

\begin{proof}
	The first term of the bound, $\sqrt{\maxeig{\Wo}}\bardx$, follows the same logic from Lemma \ref{lem:dc}, except this time $\Vert R_e(\cdot) \Vert \le \sqrt{\maxeig{\Wo}}$. The second term follows from first combining the final line of \eqref{eq:dq_derivation} with one of the three error bounds $\bar\err_{\{1,2,3\}}$, and substituting that combination into the integrand of \eqref{eq:re} (which here simplifies to $\Vert R_e \dq(t) \Vert$). Specifically, we have the following:
	\begin{equation*}
		\begin{array}{ll}
			\Vert R_e \frac{1}{2}\rho \Mo \Cred^\top (\hinv(y, \theta) - \Cred x ) \Vert &= \frac{\rho}{2} \Vert R_e \Mo \Cred^\top (\hinv(y, \theta) - \Cred x ) \Vert \\
			&=\frac{\rho}{2} \Vert R_e^{-\top } \Cred^\top (\hinv(y, \theta) - \Cred x ) \Vert \\
			&\le \frac{\rho}{2} \Vert R_e^{-\top} \Cred^\top \Vert \Vert \hinv(y, \theta) - \Cred x  \Vert \\
			&\le \frac{\rho}{2}\maxeig{\Mo}^{1/2}\Big(L_{\hinv}\sqrt{\maxsing{B_y}} \bardy + \bar\err_{\{1,2,3\}}(x,\theta)\Big)
		\end{array}
	\end{equation*}
	
	The simplification in the second equality is done via properties of the Cholesky decomposition: i.e., $R_e\Mo = R_e(R_e^\top R_e)^{-1} = R_e R_e^{-1} R_e^{-\top} = R_e^{-\top}$. The final inequality follows from $\Vert R_e^{-\top}\Vert \le \sqrt{\maxeig{\Mo}}$, $\Vert \Cred \Vert \le 1$, and applying the triangle inequality.
\end{proof}

\begin{theorem}[From derivative to value]
Let RHS denote the right hand side of \eqref{eq:vector_inequality}. Given bounds on the Riemannian distances at $t=0$: $\distc(0) \le \bar d_c(0)$ and $\diste(0) \le \bar d_e(0)$, upper bounds $\bar\distc(t) \ge \distc(t)$ and $\bar\diste(t) \ge \diste(t)$ for all $t \in [0, T]$ can be written as
	\begin{equation*}\small
		\begin{bmatrix}\distc(t) \\ \diste(t) \end{bmatrix} \le \int_{\tau=0}^{t} \textrm{RHS}\Big(\tau,\begin{bmatrix}\distc \\ \diste \end{bmatrix}\Big) d\tau \doteq \begin{bmatrix}\bar\distc(t) \\ \bar\diste(t)\end{bmatrix}, \quad \distc(0) = \bar d_c(0),\ \diste(0) = \bar d_e(0).
	\end{equation*}
\end{theorem}
\begin{proof}
	Using a vector-valued comparison theorem \cite[Corollary 1.7.1]{differential_book}, we have that $\distc(t)$ and $\diste(t)$ can be upper bounded on $[0,T]$ by the solution to the upper bound of \eqref{eq:vector_inequality}, i.e., $[\dot \distc, \dot \diste]^\top = $ RHS, provided that the solution to RHS exists until at least $t = T$. Further prerequisites for applying \cite[Corollary 1.7.1]{differential_book} are that RHS is quasi-monotone nondecreasing in $[\distc, \diste]^\top$ and that RHS is continuous in $t$ and $[\distc, \diste]^\top$. In the following, we show that these requirements hold when using $\bar\err_1$, $\bar\err_2$, and a modified, continuous version of $\bar\err_3$.
	
	First, we describe this modification to $\bar\err_3(x^*(t), \theta)$, which we note is not continuous in $t$, in general. This is because as $t$ changes, the max term in \eqref{eq:err_bnd_3} can discontinuously change, based on the datapoints that fall inside $\Omega_c(t)$. However, we can obtain a smooth upper bound to $\bar\err_3(x^*, \theta)$ by anticipating these discontinuous changes during planning (we can do this since we know the nominal dynamics and the dataset), and smoothing out these discontinuities, e.g., as in red in Fig. \ref{fig:continuous}.
	
\vspace{5pt}
\begin{figure}
\includegraphics[width=0.5\textwidth]{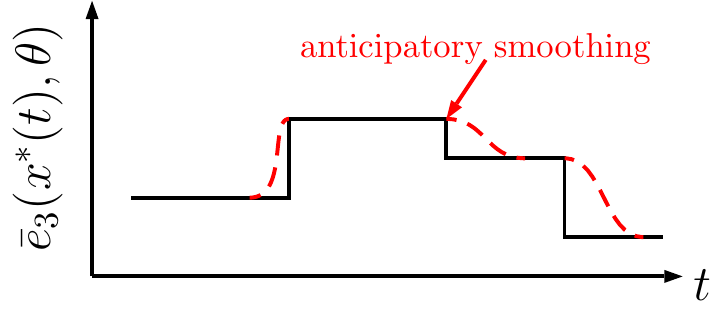}
\caption{An example of anticipatory smoothing: the red curve is a continuous upper bound to the potentially discontinuous $\bar\err_3(x^*,\theta)$.} \label{fig:continuous}
\end{figure}\vspace{5pt}
	
	We can see that RHS is continuous for all $[\distc, \diste]^\top$ (since it is just a linear system). Furthermore, it is continuous in $t$, since all terms in RHS are constant, apart from the error bounds $\bar\err_i$, which are continuous functions of $t$ (using smoothing for $\bar\err_3$). Moreover, using continuity of RHS over $t \in [0,T]$, the solution to RHS exists over $[0,T]$. A vector-valued function $g(t,r): \mathbb{R}\times\mathbb{R}^{N}\rightarrow \mathbb{R}^{N}$ is quasi-monotone nondecreasing in $r$ if $\frac{\partial g_i(t,r)}{\partial r_j} \ge 0$, for all $r$ and for all $j \ne i$, $i = 1,\ldots,N$. We can see that RHS precisely satisfies these conditions, as the matrix in \eqref{eq:vector_inequality} is Metzler, i.e., all of its off-diagonal components are nonnegative: $L_{\Delta k}, L_p, \mineigc{\Mc}, \maxeig{\Mo}, \rho \ge 0$. More precisely, $\frac{\partial \textrm{RHS}_1(t,[\distc, \diste]^\top)}{\partial \diste} = L_{\Delta k} \ge 0$ and $\frac{\partial \textrm{RHS}_2(t,[\distc, \diste]^\top)}{\partial \distc} = (*) \ge 0$.

\end{proof}

\begin{theorem}[\lmtd correctness]
		Assume that $L_{\Delta k}$, $L_{\hinv}$, and the estimated constants in $\bar\err_{\{1,2,3\}}$ are valid over their computed domains. Then Alg. \ref{alg:rrt} returns a trajectory $(x^*(t), u^*(t))$, which when tracked on the true system \eqref{eq:system} using $u(\xhat, x^*, u^*)$ with state estimates $\xhat$ generated by \eqref{eq:observer_reduced}, reaches $\mathcal{G}$ while satisfying $x(t) \in \Xsafe$, for all $t \in [0, T]$.
\end{theorem}
\begin{proof}
	By construction, Alg. \ref{alg:rrt} returns a plan $(x^*(t), u^*(t))$ which ensures that $ \Omega_c(t) \cap \Xunsafe = \emptyset$, for all $t \in [0, T]$, where $\Omega_c(t) = \{x \mid \distc(x^*(t), x(t)) \le \bar\distc(t)\}$, and $x(T) \in \mathcal{G}$. Thus, $x(t) \cap \Xunsafe = \emptyset$, provided that $x(t) \in \Omega_c(t)$, for all $t \in [0,T]$, i.e., the tubes are valid. Using Theorem \ref{thm:upper_bound} and the assumption that all estimated constants are valid over their computed domains (from the theorem statement), the tubes are valid if $x(t)$ remains in the corresponding domains of validity for all constants in \eqref{eq:vector_inequality}. In the following, we exhaustively prove that the additional constraints introduced by Alg. \ref{alg:rrt} enforce this domain invariance.
		
	First, we list all constants in \eqref{eq:trk_bnd} with a non-trivial domain of validity.
	\begin{enumerate}
		\item CCM-related constants: the minimum and maximum eigenvalues of $\Mc$: $\mineigc{\Mc}$, $\maxeigc{\Mc}$, and the CCM contraction rate $\lambda_c$.
		\item OCM-related constants: the minimum and maximum eigenvalues of $\Mo$: $\mineig{\Mo}$, $\maxeig{\Mo}$, the OCM contraction rate $\lambda_e$, and the multiplier $\rho$.
		\item Constants estimated with probabilistic correctness guarantees using the Fisher-Tippett-Gnedenko (FTG) theorem \cite{cdc, lipschitz_ral, weng2018evaluating}: $L_{\Delta k}$, $L_{\hinv}$, and all constants involved in the perception error bounds $\bar\err_i(x^*,\theta)$, $i=1,2,3$: $\bar\err_1$, $L_p$, and $\disp$.
	\end{enumerate}
	
	Now, we prove why each group of constants is valid under the constraints of Alg. \ref{alg:rrt}:
	\begin{enumerate}
		\item Suppose $\Mc$ is constant; then, its maximum and minimum eigenvalues trivially extend for all of $\X$. If $\Mc$ is polynomial and is generated via the SoS program in \eqref{eq:ccm_opt}, its eigenvalues are by construction bounded by $1/\underline{\beta}$ and $1/\bar\beta$; these provide globally-valid eigenvalue bounds over $\X$. Moreover, the CCM-based controller contracts the nominal dynamics at rate $\lambda_c$ if 1) for all time $t \in [0,T]$, the geodesic connecting $x^*(t)$ and $x(t)$ (denoted as $\gamma_c^t(s), s \in [0, 1]$) is contained within the contraction domain $D_c$, and 2) there exists a feasible control input which achieves the $\lambda_c$ contraction rate. Alg. \ref{alg:rrt} enforces both of these conditions: 1) line 9 of Alg. \ref{alg:rrt} ensures $\Omega_c(t) \subseteq D_c$, since $\Omega_c$ is defined as a sublevel set with respect to the Riemannian metric induced by $\Mc$, i.e., $\gamma_c^t(s) \subseteq \Omega_c(t) \subseteq D_c$; 2) by \eqref{eq:con_ccm}, the feedback controller $u(\xhat, x^*, u^*)$ is always feasible (i.e., \eqref{eq:ufeas} is nonempty) as long as $\xhat \in D_c$; this is enforced by line 10 of Alg. \ref{alg:rrt}.
		\item In this work, we use constant $\Mo$, so its maximum and minimum eigenvalues trivially extend for all of $\X$. The OCM-based observer contracts at rate $\lambda_e$ for its nominal dynamics if for all time $t \in [0,T]$, the geodesic connecting $x(t)$ and $\xhat(t)$ (denoted as $\gamma_e^t(s), s \in [0, 1]$) is contained within the contraction domain $D_e$. In Alg. \ref{alg:rrt}, line 10 ensures $\Omega_e(t) \subseteq D_e$, since $\Omega_e$ is defined as a sublevel set with respect to the Riemannian metric induced by $\Wo$, $\gamma_e^t(s) \subseteq \Omega_e(t) \subseteq D_e$.
		\item From its definition in the paragraph before \eqref{eq:lip}, $L_{\Delta k}$ is valid if $\distc(t) \le \bar c$ and $\diste(t) \le \bar e$, and $x^* \in D_x$, $u^* \in \U$. We ensure this via the check in Alg. \ref{alg:rrt}, line 8. $L_{\hinv}$, and the error bound constants $\bar\err_1$, $L_p$, and $\disp$ are valid if $x \in D_r$; this is checked in line 9 of Alg. \ref{alg:rrt}.
	\end{enumerate}
		
\end{proof}

\begin{remark}[Overall correctness probability]\label{rem:overall_correctness}
	If all FTG-estimated constants are obtained according to Rem. \ref{rem:lipschitz_estimation} with probability $\rho$ of correctness, using samples which are mutually independent, then the overall probability of the correctness of \lmtd can be bounded by the product of each constant being estimated correctly; for example, if $L_{\Delta k}$, $L_{\hinv}$, and $\bar\err_1$ are estimated, then \lmtd returns a safely-trackable trajectory with probability at least $\rho^3$.
\end{remark}

\section{Optimizing CCMs and OCMs for output feedback}\label{sec:app_method_ccms_and_ocms}

In this section, we describe in more detail our method for synthesizing optimized CCMs and OCMs (described briefly in Sec \ref{sec:method_ccms_and_ocms}. To implement the tracking feedback controller and observers needed at runtime, we need to obtain the CCMs and OCMs that define them. Two popular ways for synthesizing contraction metrics are convex optimization (SoS programming) \cite{manchester} and learning \cite{cdc}. While the optimization-based methods are more efficient and easier to analyze than the learning-based approaches, the learning-based methods can be applied to higher-dimensional, non-polynomial systems. Since we focus on (approximately) polynomial systems in the results, we obtain CCMs and OCMs via SoS programming; however, our method is agnostic to how the CCMs/OCMs are generated, as long as all the associated constants (e.g., $\maxeige{\Mo}, \maxeigc{\Mc}$, etc.) can be accurately bounded. 

To minimize conservativeness in planning (cf. Sec. \ref{sec:method_ofmp}), it is important that the contraction metrics induce small tubes. However, exactly optimizing the tubes in \eqref{eq:trk_bnd} is challenging, as they are coupled and depend on plan-dependent constants. For tractability, we optimize a surrogate objective and design the CCM and OCM separately.

Denote $-G_c(x)$ as the LHS of \eqref{eq:strong_con}, $v$ as a vector of indeterminates of compatible dimension with $G_c(x)$, and $\bar\beta \ge \underline\beta > 0$. Let $\I_n$ be the $n\times n$ identity matrix. For the CCM, we minimize the steady-state tracking tube width for a uniform disturbance bound as in \cite{sumeet_icra} by solving the following:
\begin{equation}\label{eq:ccm_opt}
	\begin{array}{cl}
		\underset{\Wc(x), \bar\Wc, \lambda_c, \mu_i^c(x)}{\textrm{minimize}} &\quad \frac{1}{\lambda_c}\sqrt{\frac{\maxeigc{\Wc}}{\mineigc{\Wc}}} \\
		\textrm{subject to} &\quad \eqref{eq:strong_orth} \\
		&\quad v^\top G_c(x) v - \sum_i \mu_i^e(x) c_i(x) \in \textrm{SoS}\\
		&\quad \underline\beta \I_{n_x} \preceq \Wc(x) \preceq \bar\Wc \preceq \bar\beta \I_{n_x}
	\end{array}
\end{equation}
We motivate \eqref{eq:ccm_opt} in the following. For polynomial dynamics and CCMs, the CCM condition \eqref{eq:con_ccm} is equivalent to enforcing the nonnegativity of a polynomial function $p(x)$ over a domain. A sufficient condition for proving $p(x) \ge 0$ over a domain is to enforce that $p(x)$ can be written as a sum of squares; this can be enforced via a semidefinite constraint, which is convex and is referred to in \eqref{eq:ccm_opt} as $p(x) \in \textrm{SoS}$. As it may not be possible to enforce \eqref{eq:strong_con} globally for all $x \in \X$, we further introduce nonnegative Lagrange multipliers $\mu_i^e(x)$ to only enforce \eqref{eq:strong_con} over the subset $D_c \subseteq \X$. Specifically, we rewrite $x \in D_c \Leftrightarrow \bigwedge_{i=1}^{n_\textrm{con}} \{c_i(x) \le 0\}$, where $c_i(\cdot)$ are constraint functions that together define the boundary of $D_c$; then, for some $x$ where $c_i(x) > 0$, $\mu_i^c(x)$ can take some positive value to satisfy the SoS condition even if $v^\top G(x) v < 0$. We handle condition \eqref{eq:strong_orth} by restricting $\Wc(x)$ to be only a function of a subset of state variables (cf. \cite{sumeet_icra}), and we add additional constraints on the eigenvalues of $\Wc$ to ensure that its inverse is well-defined. As in \cite{sumeet_icra}, while \eqref{eq:ccm_opt} is not convex as written, it can be solved to global optimality by solving a sequence of convex SoS feasibility programs, where a subset of the decision variables are fixed at each iteration. Specifically, we perform a line search over $\lambda_c$ and a bisection search over the condition number $\maxeigc{\Wc}/\mineigc{\Wc}$.

We take a similar approach for optimizing OCMs. First, we found it sufficient to use constant OCMs and multipliers $\rho$ for the systems in this paper due to the simplifications enabled by the inverse perception map. This simplifies the condition \eqref{eq:con_ocm} to
\begin{equation}\label{eq:con_ocm_simple}
		\Wo A(\xhat) + A(\xhat)^\top \Wo - \rho\Cred^\top \Cred + 2 \lambda_e \Wo\le 0.
\end{equation}
Let $-G_e(x)$ be the LHS of \eqref{eq:con_ocm_simple}. We then obtain our OCM by solving the following:
\begin{equation}\label{eq:ocm_opt}
	\begin{array}{cl}
		\underset{\Wo, \lambda_o, \mu_i^e(x)}{\textrm{minimize}} &\quad \frac{1}{\lambda_c}\frac{\maxeigc{\Wc}}{\sqrt{\mineigc{\Wc}}}\rho \\
		\textrm{subject to} &\quad v^\top G_e(x) v - \sum_i \mu_i^e(x) c_i^e(x) \in \textrm{SoS}\\
		&\quad \underline\beta \I_{n_x} \preceq \Wo \preceq \bar\beta \I_{n_x}
	\end{array}
\end{equation}
We change the objective function here to account for the known components of the disturbance bound; specifically, from \eqref{eq:de} we know that any disturbance from the innovation term in the observer will be pre-multiplied by $\frac{1}{2}\rho\Mo$; this accounts for the extra factors in the objective. As before, we introduce Lagrange multipliers $\mu_e(x)$ to only enforce \eqref{eq:con_ocm_simple} over the subset $D_e \subseteq \X$, and restrict the eigenvalues of $\Wo$ to ensure its inverse is well-defined. To solve \eqref{eq:ocm_opt} approximately, we drop the extra square root factor, and do the same line search on $\lambda_e$ and bisection search on the condition number as for the CCM, but instead minimizing $\rho$ instead of solving a feasibility problem in each SoS program. We note that for both CCM and OCM synthesis, if the system is linear, i.e., $\dot x = Ax + Bu$, then $\Wc$ and $\Wo$ can be chosen as constant, simplifying both \eqref{eq:ccm_opt} and \eqref{eq:ocm_opt} to a standard semidefinite program (SDP), since all constraints which were previously polynomial functions of $x$ simplify to constant linear matrix inequalities (LMIs).

\section{System models}

In this appendix, we provide an overview of the system models that we use in the results Sec. \ref{sec:results}.

\subsubsection{4D nonholonomic car:}

\begin{equation}\label{eq:car}
	\begin{bmatrix}\dot{p}_x \\ \dot{p}_y \\ \dot\phi \\ \dot v\end{bmatrix} = \begin{bmatrix}v\cos(\phi) \\ v\sin(\phi) \\ 0 \\ 0\end{bmatrix} + \begin{bmatrix} 0 & 0 \\ 0 & 0 \\ 1 & 0 \\ 0 & 1 \end{bmatrix}\begin{bmatrix}\omega \\ a\end{bmatrix} + \begin{bmatrix} 0 & 0 \\ 0 & 0 \\ 1 & 0 \\ 0 & 1 \end{bmatrix}\dx
\end{equation}
where $p_x$ and $p_y$ are the $x$ and $y$ translations of the car, $\phi$ is its orientation, and $v$ is its linear velocity. where $u=[\omega,a]^\top$. To make \eqref{eq:car} compatible with SoS programming, when searching for the CCM and OCM, we polynomialize the dynamics by fitting a degree 5 polynomial to $\sin(\cdot)$ and $\cos(\cdot)$ over $[-\pi/2, \pi/2]$.

\subsubsection{6D planar quadrotor:}

We use the dynamics from \cite[p.20]{sumeet_icra} with six states and two inputs:
\begin{equation}\label{eq:pvtol}
	\begin{bmatrix}\dot{p}_x \\ \dot{p}_y \\ \dot{\phi} \\ \dot{v}_x \\ \dot{v}_z \\ \ddot{\phi} \end{bmatrix} = \begin{bmatrix} v_x \cos(\phi) - v_z \sin(\phi) \\ v_x \sin(\phi) + v_z \cos(\phi) \\ \dot\phi \\ v_z \dot\phi - g\sin(\phi) \\ -v_x \dot\phi - g\cos(\phi) \\ 0\end{bmatrix} + \begin{bmatrix}0 & 0 \\ 0 & 0 \\ 0 & 0 \\ 0 & 0 \\ 1/m & 1/m \\ l/J & -l/J\end{bmatrix}\begin{bmatrix}u_1 \\ u_2\end{bmatrix} + \begin{bmatrix}0 & 0 \\ 0 & 0 \\ 0 & 0 \\ 0 & 0 \\ 1 & 0 \\ 0 & 1\end{bmatrix}\dx,
\end{equation}
where $x = [p_x, p_z, \phi, v_x, v_z, \dot\phi]$ models the linear/angular position and velocity, and $u=[u_1,u_2]$ models thrust. We use the parameters $m = 0.486$, $l = 0.25$, and $J = 0.07$. To make \eqref{eq:pvtol} compatible with SoS programming, when searching for the CCM and OCM, we polynomialize the dynamics by fitting a degree 5 polynomial to $\sin(\cdot)$ and $\cos(\cdot)$ over $[-\pi/2, \pi/2]$.

\subsubsection{17D manipulation task:}

Let the $m\times n$ zero matrix be denoted $\mathbf{0}_{m\times n}$. We define $\bm{\phi} = [\phi_1, \phi_2, \phi_3]^\top$, $\mathbf{j} = [j_1, j_2, \ldots, j_7]^\top$, and $\mathbf{\dot j} = [\dot j_1, \dot j_2, \ldots, \dot j_7]^\top$. We model the full 17D system as:
\begin{equation}\label{eq:arm_17d}
	\begin{bmatrix}
		\bm{\dot\phi} \\
		\mathbf{\dot j} \\
		\mathbf{\ddot j}
	\end{bmatrix} = \begin{bmatrix}
		\mathbf{0}_{3\times 3} & \mathbf{0}_{3\times 7} & \mathbf{0}_{3\times 7} \\
		\mathbf{0}_{7 \times 3} & \mathbf{0}_{7 \times 7} & \I_{7 \times 7} \\
		\mathbf{0}_{7 \times 3} & \mathbf{0}_{7 \times 7} & \mathbf{0}_{7 \times 7}
	\end{bmatrix}\begin{bmatrix}
		\bm{\phi} \\
		\mathbf{j} \\
		\mathbf{\dot j}
	\end{bmatrix} + \begin{bmatrix}
		\mathbf{0}_{3\times 7} \\
		\mathbf{0}_{7\times 7} \\
		\I_{7 \times 7}
	\end{bmatrix}u + \begin{bmatrix}
		\mathbf{0}_{3\times 7} \\
		\mathbf{0}_{7\times 7} \\
		\I_{7 \times 7}
	\end{bmatrix}\dx
\end{equation}

Here, the $\phi_i$ are the Euler angles of the object relative to the end effector, the $j_i$ are the Kuka joint angles, the $\dot j_i$ are the Kuka joint velocities, and $u$ are the commanded joint accelerations.

The 14D subsystem referred to in the text is:
\begin{equation}\label{eq:arm_14d}
	\begin{bmatrix}
		\mathbf{\dot j} \\
		\mathbf{\ddot j}
	\end{bmatrix} = \begin{bmatrix}
		\mathbf{0}_{7 \times 7} & \I_{7 \times 7} \\
		\mathbf{0}_{7 \times 7} & \mathbf{0}_{7 \times 7}
	\end{bmatrix}\begin{bmatrix}
		\mathbf{j} \\
		\mathbf{\dot j}
	\end{bmatrix} + \begin{bmatrix}
		\mathbf{0}_{7\times 7} \\
		\I_{7 \times 7}
	\end{bmatrix}u
\end{equation}

\end{document}